\documentclass{article}

     \PassOptionsToPackage{numbers, compress}{natbib}
 \usepackage[preprint]{neurips_2024}

\usepackage[utf8]{inputenc} % allow utf-8 input
\usepackage[T1]{fontenc}    % use 8-bit T1 fonts
\usepackage{hyperref}       % hyperlinks
\usepackage{url}            % simple URL typesetting
\usepackage{booktabs}       % professional-quality tables
\usepackage{amsfonts}       % blackboard math symbols
\usepackage{nicefrac}       % compact symbols for 1/2, etc.
\usepackage{microtype}      % microtypography
\usepackage{xcolor}         % colors

\usepackage{graphicx}
\usepackage{amsmath} 
\usepackage{amsthm} % added 
\theoremstyle{definition}
\newtheorem{dfn}{Definition}[section]
\newtheorem{lemma}{Lemma}[section]

\usepackage{color} % added by kfukui

\newcommand{\diag}{\mathsf{diag}}
\newcommand{\Gr}{\mathsf{Gr}}
\newcommand{\Span}{\mathsf{span}}
\newcommand{\sinc}{\mathsf{sinc}}
\newcommand{\Trace}{\mathsf{Trace}}
\newcommand{\Mag}{\mathsf{Mag}}

\title{Second-order difference subspace}

\author{%
  Kazuhiro Fukui, Pedro H.V. Valois\\  
  Department of Computer Science,  University of Tsukuba, Japan\\
  \And
  Lincon Souza, Takumi Kobayashi\\
  National Institute of Advanced Industrial Science and Technology (AIST), Japan \\
}

\begin{document}

\maketitle

\begin{abstract}
Subspace representation is a fundamental technique in various fields of machine learning. Analyzing a geometrical relationship among multiple subspaces is essential for understanding subspace series' temporal and/or spatial dynamics. This paper proposes the second-order difference subspace, a higher-order extension of the first-order difference subspace between two subspaces that can analyze the geometrical difference between them. As a preliminary for that, we extend the definition of the first-order difference subspace to the more general setting that two subspaces with different dimensions have an intersection. We then define the second-order difference subspace by combining the concept of first-order difference subspace and principal component subspace (Karcher mean) between two subspaces, motivated by the second-order central difference method. We can understand that the first/second-order difference subspaces correspond to the velocity and acceleration of subspace dynamics from the viewpoint of a geodesic on a Grassmann manifold. We demonstrate the validity and naturalness of our second-order difference subspace by showing numerical results on two applications: temporal shape analysis of a 3D object and time series analysis of a biometric signal.
\end{abstract}

\section{Introduction}
Subspace representation appears as a critical component in various types of research in machine learning\cite{watanabe,fukunaga2,oja,parsons2004subspace,katayama2005subspace,fg98,kim1,cmsm, SSAbook,Subspace_methods}.
Depending on the application, subspace representations offer different advantages: a shape subspace allows invariance to rotation for skeleton data in tasks such as hand gesture and action recognition\cite{shapesp2010,shapesp}; in signal processing, a signal subspace generated from spectral analysis captures oscillatory components with complex frequency structure while providing invariance to changes in amplitude\cite{SSAbook,gssa,tgssa,mssa,icassp2024,maha2024}; and in natural language processing, a word subspace conveys the leading concepts in a piece of text\cite{erica2018,erica2021}; and subspaces can arise in many more applications where the data of interest has a low-dimensional structure represented by a high-dimensional feature space.

Analyzing temporal and/or spatial dynamics of the series of such subspaces is a central technique in machine learning. To construct such an analysis, we introduced the first-order difference subspace to represent the difference between two subspaces as a natural generalization of a difference vector between two vectors \cite{tpami2015,tpami2023}. The first-order difference subspace effectively works as valid feature extraction based on the first-order differential operator in various tasks such as shape analysis \cite{cmsm,tpami2015,tpami2023} and anomaly detection from the time series \cite{abnormal_detection}.
In this paper, we extend the concept of the first-order difference subspace to the second one and then define it corresponding to the second-order differential operator. By incorporating the second difference subspace, we significantly enhance our subspace analysis framework to advance the research based on subspace representation.

As a preparation for the definition of second-order difference, we first review the mechanism of the first-order difference subspace (DS)\cite{tpami2015}, defined in the simple case that two subspaces with the same dimension have no intersection. We then extend the definition to a general case in which two subspaces with different dimensions intersect. Our definition of the second-order difference subspace is motivated by the second-order central difference method for solving various differential equations. It works as the second-order differential operator corresponding to the first-order difference subspace.

Given three sequential subspaces, ${\mathcal{S}}_1$, ${\mathcal{S}}_2$ and ${\mathcal{S}}_3$, we naturally define the second-order DS as the first-order DS between the subspace ${\mathcal{S}}_2$ and the principal component subspace (Karcher mean) ${\mathcal{M}}$ between ${\mathcal{S}}_1$ and ${\mathcal{S}}_3$. We can understand the physical meaning of this definition by considering a geodesic on a Grassmann manifold, where the first-order and second-order DSs correspond to the velocity and acceleration of a moving subspace, respectively.

We demonstrate the validity and naturalness of our second-order DS on two numerical experiments: temporal analysis of deforming 3D shape and time series analysis of biometric signals. In the first experiment, the 3D shape of a walking/jumping subject is represented by a series of three-dimensional subspaces. Each subspace models the shape at one time step, forming a series as the shape evolves through time \cite{shapesp2010,shapesp}. We show that the first/second-order DS can capture the walking action's expected velocity and acceleration. In the second one, a short duration signal is represented by a series of high dimensional subspaces known as signal subspace, each one representing an overlapping window of the signal\cite{SSAbook}.
The first/second-order DSs also show consistent behaviour in this scenario, although this setting is more complicated because there is a large intersection between sequential signal subspaces.

The contributions of this paper are summarized as follows:
\begin{itemize}
\item We extend the definition of the first-order difference subspace to a general case where two subspaces with different dimensions intersect.
\item We propose the second-order difference subspace as a natural extension of the first-order difference subspace.
\item We demonstrate our definition's validity and physical naturalness through numerical experiments.
\end{itemize}
The remainder of the paper is organized as follows. In Section \ref{sec:DS}, we review the definition of first-order DS in the simple case and extend it to a general case. In Section \ref{sec:2ndDS}, we define the second-order DS. In Section \ref{sec:geometry}, we discuss the geometry of second-order DS. In Section \ref{sec:numexps}, we demonstrate the validity of the first/second-order difference subspaces. %In Section \ref{sec:limitations}, we discuss the limitations of the definition. 
We conclude in Section \ref{sec:conclusions}.

\section{First-order difference subspace}
\label{sec:DS}
First-order difference subspace (DS) ${\mathcal{D}}$ between two subspaces, $\mathcal{S}_1$ and $\mathcal{S}_2$, is a generalization of the difference vector $\mathbf{d}$ between two  vectors, $\mathbf{u}$ and $\mathbf{v}$, with unit length, as shown in Fig~\ref{fig:concept-ds}. The first-order DS ${\mathcal{D}}(\mathcal{S}_1,\mathcal{S}_2)$ between $\mathcal{S}_1$ and $\mathcal{S}_2$ can be defined in both geometrical and analytical viewpoints \cite{tpami2015}. 

The first-order DS is originally defined in a simple case in which there is no intersection subspace between two subspaces with the same dimension \cite{tpami2015}. In this paper, we remove this limitation and provide the definition in the more general case that there is a $r$-dimensional intersection subspace between two subspaces with different dimensions. 

\begin{figure}[tb]
\begin{minipage}[t]{0.48\columnwidth}
\centering
\includegraphics[scale=0.35]{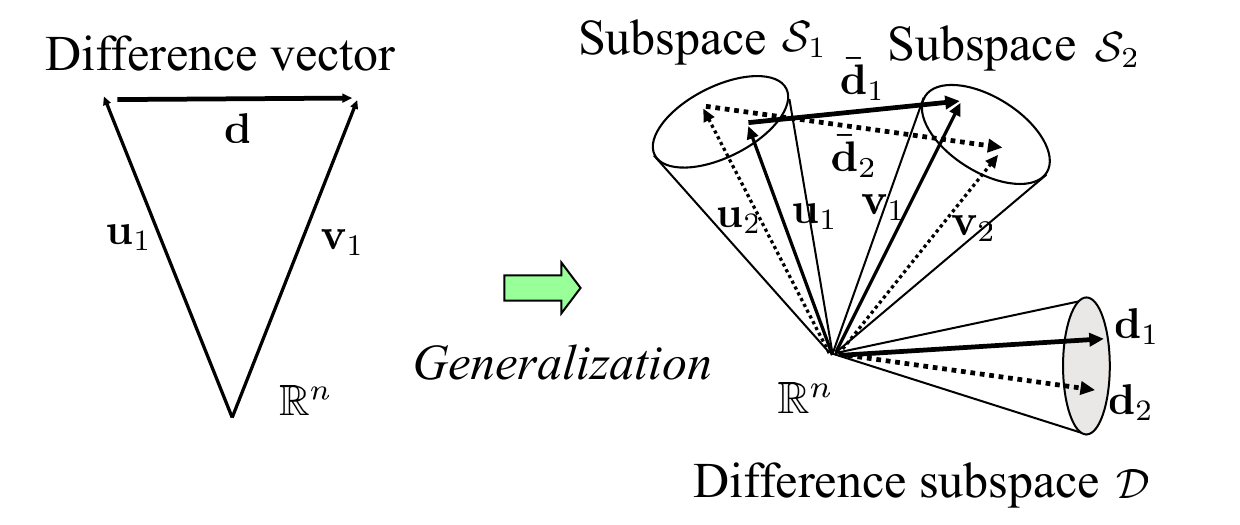}
%\label{fig:DS}
(a) Geometrical definition
\end{minipage}
\hspace{8mm}
\begin{minipage}[t]{0.48\columnwidth}
\centering
\includegraphics[scale=0.26]{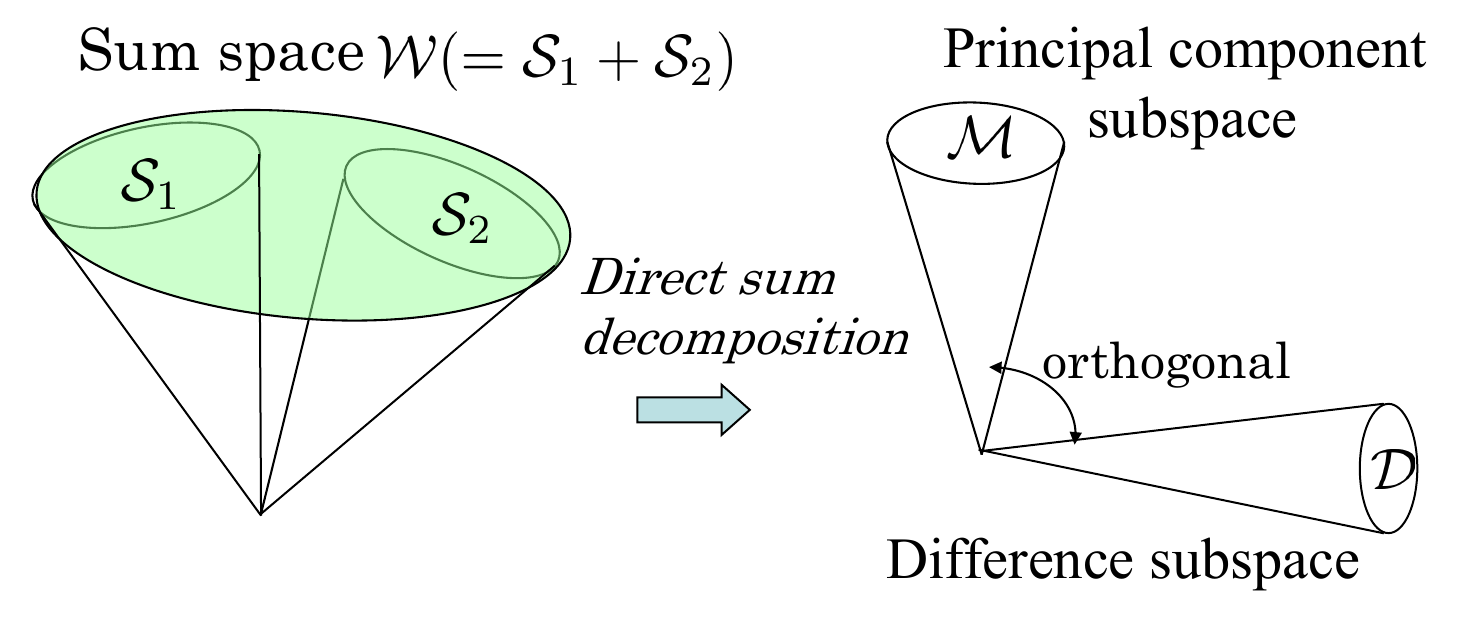}
%\caption{First-order difference subspace $\mathcal{D}$ between subspaces $\mathcal{S}_1$ and $\mathcal{S}_2$.} \label{fig:concept-ds}
%\label{fig:directsum}
(b) Analytical definition
\end{minipage}
\caption{First-order difference subspace $\mathcal{D}$ between subspaces $\mathcal{S}_1$ and $\mathcal{S}_2$ \cite{tpami2015}.} 
\label{fig:concept-ds}
\label{fig:DS}
\end{figure}

\subsection{Definitions in the simple case}
\noindent{\bf{Geometrical definition}:}
Given $d_1$-dimensional subspaces $\mathcal{S}_1$ and $\mathcal{S}_2$ in $n$-dimensional vector space, assuming that the two subspaces have no intersection, we can calculate $d_1$ canonical angles $\lbrace \theta_i \rbrace{^{d_1}_{i=1}}$  between them~\cite{cangle1,cangle2}, in an ascending order, $\theta_i\leq\theta_j\ \forall i<j$.
Let ${\bar{\mathbf{d}}}_i={\mathbf{u}}_i-{\mathbf{v}}_i \in {\mathbb{R}}^{n}$ be the difference vector between the canonical vectors ${\mathbf{u}}_i\in\mathcal{S}_1$ and ${\mathbf{v}}_i\in\mathcal{S}_2$, which form the $i$th canonical angle, $\theta_i$, as shown in Fig.\ref{fig:DS}(a). We can regard the normalized difference vectors 
$\{{\mathbf{d}}_i = \frac{{\bar{\mathbf{d}}}_i}{\|{\bar{\mathbf{d}}}_i\|}\}_{i=1}^{d_1}$ as the orthonormal basis vectors of the difference subspace $\mathcal{D}$, as they are orthogonal to each other \cite{tpami2015}.  

According to the geometrical definition, we can calculate the orthonormal basis of ${\mathcal{D}}$ with low computational cost \cite{Kobayashi_2023_BMVC}.
Let $\mathbf{\Phi} \in {\mathbb{R}}^{{n}\times{d_1}}$ and $\mathbf{\Psi} \in {\mathbb{R}}^{{n}\times{d_1}}$ be orthonormal basis of ${\mathcal{S}}_1$ and ${\mathcal{S}}_2$. The canonical angles ${\{\theta_i\}}_{i=1}^{d_1}$ can be calculated by performing the SVD as follows: 
\begin{equation}
 {\mathbf{\Phi}}^{\top} {\mathbf{\Psi}} = {\mathbf{U}} {\mathbf{\Sigma}} {\mathbf{V}}^{\top},
  \label{eq:svdMSM}
\end{equation}
where $\mathbf{U} \in {\mathbb{R}}^{{d_1}\times{d_1}}$, ${\mathbf{\Sigma}}$ is a diagonal matrix $\in {\mathbb{R}}^{{d_1}\times{d_1}}$, $\mathbf{V} \in {\mathbb{R}}^{{d_1}\times{d_1}}$ and ${\diag(\mathbf\Sigma)}=\{\cos\theta_1, \cdots, \cos\theta_{d_1}\}$.

{\color{black}
The orthonormal basis ${\mathbf{D}}$ of difference subspace ${\mathcal{D}}$ is represented in matrix form as follows \cite{Kobayashi_2023_BMVC}:
\begin{align}
{\mathbf{D}} = ({\mathbf{{\Phi}U-{\Psi}V}})\{2({\mathbf{I}}-{\mathbf{\Sigma)}}\}^{-\frac{1}{2}},
 \label{eq:svd1}
\end{align}
where ${\mathbf{{\Phi}U-{\Psi}V}}=[\bar{\mathbf{d}}_1,\bar{\mathbf{d}}_2,\cdots, \bar{\mathbf{d}}_{d_1}]$ $\in {\mathbb{R}}^{n{\times}{d_1}}$
%[{\mathbf{u}}_1-{\mathbf{v}}_1, {\mathbf{u}}_2-{\mathbf{v}}_2,\dots, {\mathbf{u}}_{d_1}-{\mathbf{v}}_{d_1}]
and $\{2({\mathbf{I}}-{\mathbf{\Sigma)}}\}^{-\frac{1}{2}}$=${\diag({\|\bar{\mathbf{d}}_1\|}_2, \dots, {\|\bar{\mathbf{d}}_{d_1}\|}_2)}^{-1}$ $\in {\mathbb{R}}^{{d_1}{\times}{d_1}}$.

As a pair to ${\mathcal{D}}$, we defined the principal component subspace ${\mathcal{M}}$ that is spanned by $\{{\mathbf{m}}_i = \frac{{\bar{\mathbf{m}}}_i}{\|{\bar{\mathbf{m}}}_i\|}\}_{i=1}^{d_1}$, where ${\bar{\mathbf{m}}}_i={\mathbf{u}}_i+{\mathbf{v}}_i \in {\mathbb{R}}^{n}$ \cite{tpami2015}. The principal component subspace is equal to the Karcher mean between two subspaces ${\mathcal{S}}_1$ and ${\mathcal{S}}_2$ \cite{Kobayashi_2023_BMVC}.
The orthonormal basis ${\mathbf{M}}$ of the principal component subspace ${\mathcal{M}}$ between two subspaces ${\mathcal{S}}_1$ and ${\mathcal{S}}_2$ is represented in matrix form as follows \cite{Kobayashi_2023_BMVC}:
\begin{align}
{\mathbf{M}} = ({\mathbf{{\Phi}U+{\Psi}V}})\{{2({\mathbf{I}} + {\mathbf{\Sigma)}}}\}^{-\frac{1}{2}},
 \label{eq:svd2}
\end{align}
where ${\mathbf{{\Phi}U+{\Psi}V}}=[\bar{\mathbf{m}}_1, \cdots, \bar{\mathbf{m}}_{d_1}]$ $\in {\mathbb{R}}^{n{\times}{d_1}}$ and  
$\{2({\mathbf{I}}+{\mathbf{\Sigma)}}\}^{-\frac{1}{2}}={\diag({\|\bar{\mathbf{m}}_1\|}_2,\dots,{\|\bar{\mathbf{m}}_{d_1}\|}_2)}^{-1}$
$\in$ ${\mathbb{R}}^{{d_1}{\times}{d_1}}$. 

For the above derivation of ${\mathbf{D}}$ and ${\mathbf{M}}$, see \ref{lemma:orthonormal basis of D and M} in the appendix.
}

\vspace{3mm}
\noindent{\bf{Analytical definition}:}
The difference subspace ${\mathcal{D}}$ can also be analytically defined by using the orthogonal projection matrices of the two subspaces  $\mathcal{S}_1$ and $\mathcal{S}_2$ \cite{cangle2}.

\if 0
The orthogonal projection matrices, $\mathbf{P}_1$ and $\mathbf{P_2}$ $\in {\mathbb{R}}^{n{\times}n}$, are defined as follows:
\begin{align}
{\mathbf{P}_1}&={\mathbf{\Phi}}{\mathbf{\Phi}}^{T},\\
{\mathbf{P}_2}&={\mathbf{\Psi}}{\mathbf{\Psi}}^{T}.
\end{align}
\fi

The basis vectors of the difference subspace $\mathcal{D}$ are calculated as the eigenvectors of the sum matrix $\mathbf{G}=\mathbf{P}_1+\mathbf{P}_2$ of the projection matrices as follows:
\begin{equation}
{\mathbf{GD}}={\mathbf{D\Lambda}},
 \label{eq:ds}
\end{equation}
where ${\mathbf{P}_1}={\mathbf{\Phi}}{\mathbf{\Phi}}^{\top}$ and ${\mathbf{P}_2}={\mathbf{\Psi}}{\mathbf{\Psi}}^{\top}$. $\mathbf{D}$ is the matrix arranging eigenvectors $\{{\mathbf{d}}_i\}$ in columns and $\diag(\mathbf{\Lambda})=[\lambda_1, \dots, \lambda_{d_1}]$.

\if 0
The basis vectors of the difference subspace $\mathcal{D}$ are calculated as the eigenvectors of the sum matrix $\mathbf{G}=\mathbf{P}_1+\mathbf{P}_2$ of the projection matrices as follows:
\begin{equation}
 {\mathbf{GD}}={\mathbf{D\Sigma}},
 \label{eq:ds}
\end{equation}
where $\mathbf{D}$ is the matrix arranging eigenvectors $\{{\mathbf{d}}_i\}$ in columns and $\mathbf{\Sigma}$ is the diagonal matrix containing eigenvalues, $\lambda_1, \dots, \lambda_{d_1}$, in the diagonal elements. 
\fi

%In this definition, $\mathcal{D}$ is spanned by the $d_1$ eigenvectors $\{{\mathbf{d}}_i\}$ of $\mathbf{G}$ corresponding to the positive eigenvalues smaller than one. The rest eigenvectors with nonzero eigenvalues define the principal component subspace $\mathcal{M}$ between the two subspaces. 
\vspace{3mm}
\begin{dfn}[Difference subspace $\mathcal{D}$] $d_1$ eigenvectors $\{{\mathbf{d}}_i^{\prime}\}_{i=1}^{d_1}$ of $\mathbf{G}$ corresponding to the positive eigenvalues smaller than one span the difference subspace ${\mathcal{D}}({\mathcal{S}_1}, {\mathcal{S}_2})$ between ${\mathcal{S}_1}$ and ${\mathcal{S}_2}$. 
\end{dfn}

\begin{dfn}[Principal component subspace (Karcher mean) $\mathcal{M}$] $d_1$ eigenvectors $\{{\mathbf{m}}_i^{\prime}\}_{i=1}^{d_1}$ of $\mathbf{G}$ corresponding to eigenvalues larger than one span the principal component subspace $\mathcal{M}({\mathcal{S}_1}, {\mathcal{S}_2})$ between ${\mathcal{S}_1}$ and ${\mathcal{S}_2}$. 
\end{dfn}

The above definitions correspond that the sum subspace $\mathcal{W}$ of ${\mathcal{S}_1}$ and ${\mathcal{S}_2}$ is orthogonally decomposed into $\mathcal{D}$ and $\mathcal{M}$  as shown in Fig.\ref{fig:concept-ds}(b). This geometrical relationship means that the difference subspace $\mathcal{D}$ is generated by subtracting $\mathcal{M}$ that represents the mean of the two subspaces from the sum subspace $\mathcal{W}$; in other words, $\mathcal{D}$ can contain only the remaining difference component as its name suggests.

Note that the eigenvectors $\{{\mathbf{d}}_i^{\prime}\}$ and $\{{\mathbf{m}}_i^{\prime}\}$ are equal to $\{{\mathbf{d}}_i\}$ and $\{{\mathbf{m}}_i\}$ defined in the geometrical definition, respectively. For the details, see \cite{tpami2015}.

\subsection{Definitions in general case}
\label{sec:general case}
We consider more general case that $d_1$-dimensional subspace ${\mathcal{S}_1}$ and $d_2$-dimensional subspace ${\mathcal{S}_2}$ have a $r$-dimensional intersection subspace in $n$-dimensional vector space, for convenience sake, $d_1 \leq d_2$.

\begin{lemma}
\label{lemma:new orthonormal basis}
The orthonormal basis ${\mathbf{\Phi}} \in {\mathbb{R}}^{{n}\times{d_1}}$ of ${\mathcal{S}_1}$ and $\mathbf{\Psi} \in {\mathbb{R}}^{{n}\times{d_2}}$ of ${\mathcal{S}_2}$ ($d_1 \leq d_2$) can be orthogonally transformed to the following orthonormal bases ${\mathbf{\Phi}}^{*} \in {\mathbb{R}}^{{n}\times{d_1}}$ and ${\mathbf{\Psi}}^{*} \in {\mathbb{R}}^{{n}\times{d_2}}$: 
%Let ${\mathcal{S}_1}, {\mathcal{S}_2} \in \mathbb{R}^n$ be subspaces of dimension $d_1$ and $d_2$ respectively, where $d_1 \leq d_2$. ${\mathcal{S}_1}$ and ${\mathcal{S}_2}$ have an $r$-dimensional intersection subspace ${\mathcal{I}}$. The following $\mathbf{\Phi}^{*}$ and $\mathbf{\Psi}^{*}$ are possible orthonormal basis for ${\mathcal{S}_1}$ and ${\mathcal{S}_2}$.
\begin{align}
\mathbf{\Phi}^{*} &=[ {\boldsymbol{\gamma}}_1, \dots, {\boldsymbol{\gamma}}_r, {\mathbf{u}}_{r+1},  \dots, {\mathbf{u}}_{d_1} ], \\
\mathbf{\Psi}^{*} &= [ {\boldsymbol{\gamma}}_1, \dots, {\boldsymbol{\gamma}}_r,  {\mathbf{v}}_{r+1},  \dots, {\mathbf{v}}_{d_1}, {\boldsymbol{\beta}}_{d_1+1}, \dots, {\boldsymbol{\beta}}_{d_2} ],\label{eq:beta}
\end{align} 
where $\{{\boldsymbol{\gamma}}_i\}$ represents the orthonormal basis of an  intersection subspace $\mathcal{I}$ between ${\mathcal{S}_1}$ and ${\mathcal{S}_2}$, $\{{\mathbf{u}}_i\}$ and $\{{\mathbf{v}}_i\}$ are a pair of the canonical vectors forming the $i$th nonzero canonical angle $\theta_i$ between 
${\mathcal{S}_1}$ and ${\mathcal{S}_2}$.
$\{{\boldsymbol{\beta}}_{d_1+1}, \dots, {\boldsymbol{\beta}}_{d_2}\}$ are orthogonal to the remaining orthonormal basis,  $\{{\boldsymbol{\gamma}}_i\}$, $\{{\mathbf{u}}_i\}$ and $\{{\mathbf{v}}_i\}$.
\end{lemma}
For the proof, see Lemma \ref{lemma:new orthonormal basis} in the appendix.

\noindent{\bf{Geometrical definition}:}
{\color{black}
From the above Lemma, only $d_1-r$ nonzero canonical angles can be defined between them \cite{cangle1, cangle2}. The $r$ remaining canonical angles corresponding to the intersection subspace are zero since the two canonical vectors are the same as ${\{{\boldsymbol{\gamma}}_{i}\}}_{i=1}^r$. Accordingly, $\mathcal{D}$ is defined as a subspace spanned by $[{\mathbf{d}}_1, {\mathbf{d}}_2, \dots, {\mathbf{d}}_{d_1-r} ]$, where ${\mathbf{d}}_i=\frac{{\mathbf{v}}_{i+r}-{\mathbf{u}}_{i+r}}{||{\mathbf{v}}_{i+r}-{\mathbf{u}}_{i+r}||}$. 
$\mathcal{M}$ is defined as a subspace spanned by $[{\mathbf{m}}_1, {\mathbf{m}}_2, \dots, {\mathbf{m}}_{d_1-r} ]$, where ${\mathbf{m}}_i=\frac{{\mathbf{v}}_{i+r}+{\mathbf{u}}_{i+r}}{||{\mathbf{v}}_{i+r}+{\mathbf{u}}_{i+r}||}$.
For Eqs.(\ref{eq:svd1}) and (\ref{eq:svd2}), $\mathbf{U} \in {\mathbb{R}}^{{d_1}\times{(d_1-r)}}$, ${\mathbf{\Sigma}}$ is a diagonal matrix $\in {\mathbb{R}}^{{(d_1-r)}\times{(d_1-r)}}$, $\mathbf{V} \in {\mathbb{R}}^{{(d_1)}\times{(d_1-r)}}$ and ${\diag(\mathbf\Sigma)}=\{\cos\theta_1, \dots, \cos\theta_{d_1-r}\}$.
}

\vspace{2mm}
\noindent{\bf{Analytical definition}:}
Given the sum matrix $\mathbf{G}=\mathbf{P}_1+\mathbf{P}_2$, where $\mathbf{P}_1$=${\mathbf{\Phi}}{{\mathbf{\Phi}}}^{\top}$=${{\mathbf{\Phi}}^{*}}{{\mathbf{\Phi}}^{*}}^{\top}$ and $\mathbf{P}_2$=${\mathbf{\Psi}}{{\mathbf{\Psi}}}^{\top}$=${{\mathbf{\Psi}}^{*}}{{\mathbf{\Psi}}^{*}}^{\top}$ are the orthogonal projection matrix onto each subspace, respectively. 
Lemma \ref{lemma:new orthonormal basis} leads us the definitions of difference subspace ${\mathcal{D}}$ and principal component subspace ${\mathcal{M}}$ in the general case.
For the proof, see Lemma \ref{lemma:eigenvalues of G} in the appendix 

\vspace{2mm}
\begin{dfn}[Difference subspace $\mathcal{D}$] $d_1-r$ eigenvectors of $\mathbf{P}_1+\mathbf{P}_2$ corresponding to nonzero eigenvalues smaller than one span the difference subspace ${\mathcal{D}}({\mathcal{S}_1}, {\mathcal{S}_2})$ between ${\mathcal{S}_1}$ and ${\mathcal{S}_2}$.  
\end{dfn}
\vspace{1mm}
\begin{dfn}[Principal component subspace $\mathcal{M}$] $d_1$ eigenvectors of $\mathbf{P}_1+\mathbf{P}_2$ corresponding to eigenvalues larger than one span the principal component subspace $\mathcal{M}({\mathcal{S}_1}, {\mathcal{S}_2})$ between ${\mathcal{S}_1}$ and ${\mathcal{S}_2}$. 
\end{dfn}
\vspace{1mm}
\begin{dfn}[Intersection subspace $\mathcal{I}$] $r$ eigenvectors of $\mathbf{P}_1+\mathbf{P}_2$, corresponding to the eigenvalues equal to 2, span the intersection subspace $\mathcal{I}({\mathcal{S}_1}, {\mathcal{S}_2})$ between ${\mathcal{S}_1}$ and ${\mathcal{S}_2}$.  
\end{dfn}

A subspace $\mathcal{Z}$ spanned by ${d_2}-{d_1}$ eigenvectors, $\{{\boldsymbol\beta}_i\}_{i=d_1}^{d_2}$ in Eq.(\ref{eq:beta}), of $\mathbf{P}_1+\mathbf{P}_2$ corresponding to the eigenvalues equal to one are not included in either of ${\mathcal{M}}$ and ${\mathcal{D}}$. The sum subspace $\mathcal{W}$ is orthogonally decomposed into three subspaces, principal component subspace $\mathcal{M}$, difference subspace $\mathcal{D}$ and the subspace $\mathcal{Z}$ as shown in Fig. \ref{fig:orthogonal decomp}. The intersection subspace $\mathcal{I}$ is included in the principal component subspace $\mathcal{M}$. 

The eigenvectors corresponding to eigenvalues near to one and zero are unstable in the above definitions. Thus, in practice, we use only the eigenvectors corresponding to the eigenvalues smaller than 1.0-$\delta$ and larger than a given small positive value $\delta$, where $\delta$ is set in for example, 1e-3 $\sim$ 1e-6. 

\begin{figure}[tb]
\centering
 \begin{minipage}[b]{0.45\columnwidth}
\includegraphics[scale=0.29]{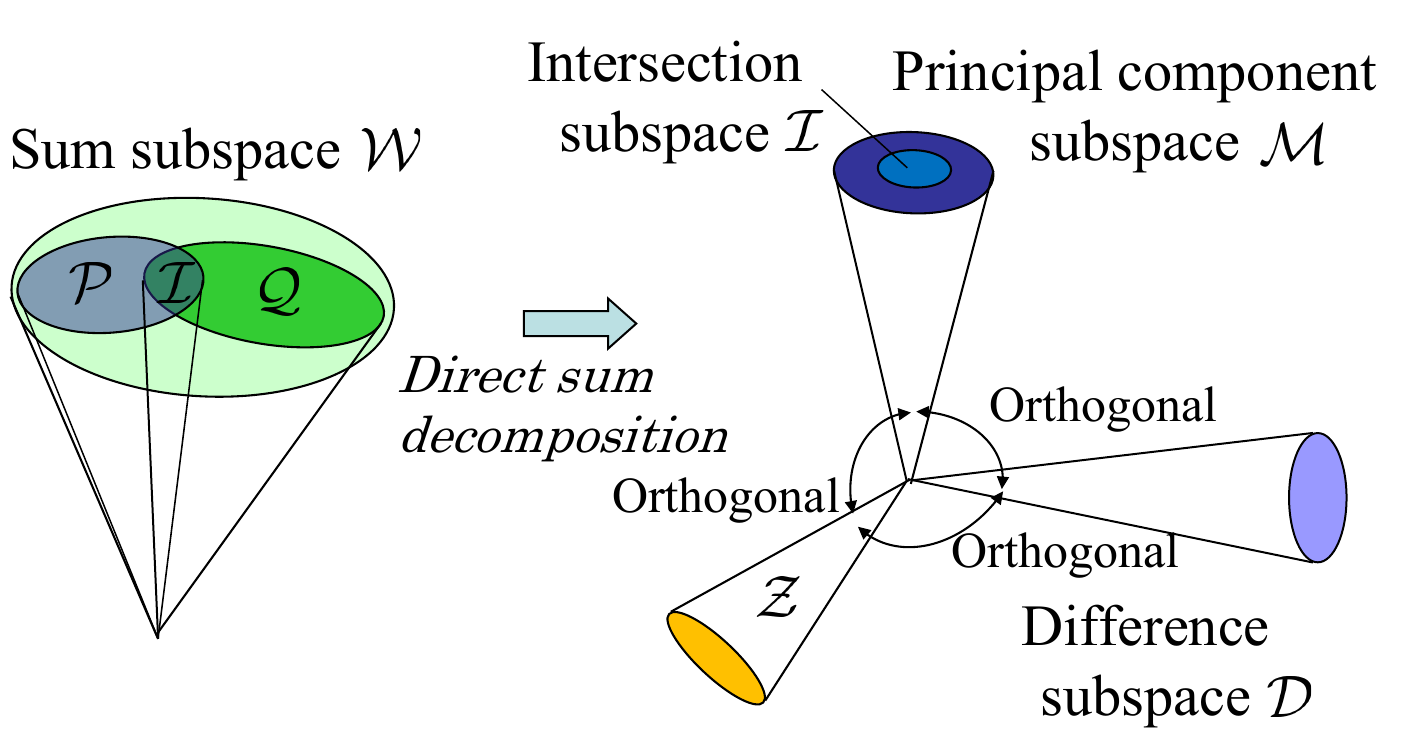}
\caption{Orthogonal decomposition of sum subspace $\mathcal{W}$ into three subspaces.} 
\label{fig:orthogonal decomp}
\end{minipage}
\hspace{5mm}
\begin{minipage}[b]{0.45\columnwidth}
  \centering
  \includegraphics[scale=0.29]{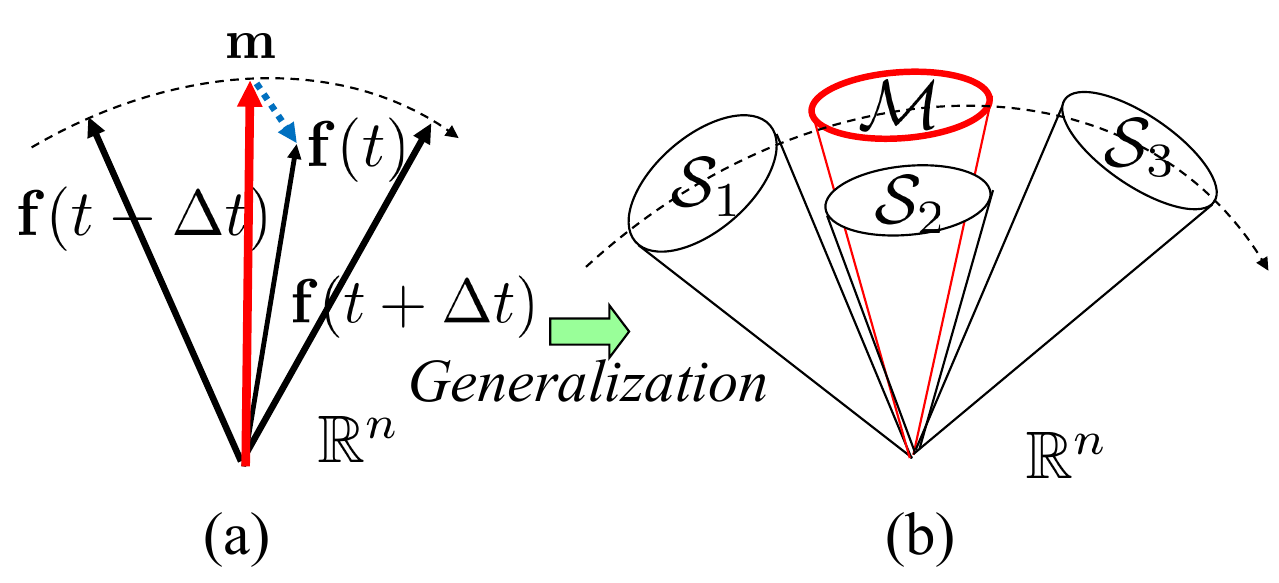}
  \caption{Three sequential vectors/subspaces in $n$-dimensional vector space.} 
  \label{fig:three subspaces}
 \end{minipage}
\end{figure}

\subsection{Magnitude of difference subspace}
The difference subspace (DS) has information solely about {\it direction}. Thus, we introduce the magnitude of DS to measure the distance between two subspaces like length of different vector ${||{\mathbf{u}}-{\mathbf{v}}||}_2^2$.
\vspace{1mm}
\begin{dfn}[Magnitude of difference subspace] The magnitude of DS $\mathcal{D}$ between two subspaces, $\mathcal{S}_1$ and $\mathcal{S}_2$, is formed by
\begin{displaymath}
    \Mag({\mathcal{D}}(\mathcal{S}_1,\mathcal{S}_2)) \stackrel{\mathrm{def}}{=} \Trace(2({\mathbf{I}}-{\mathbf{\Sigma)}})=\sum_{i=1}^{d_1} {\|{\mathbf{\bar{d}}}_{i}\|}_2^2,
\end{displaymath}
\label{eq:mag1}    
where $\mathbf{\Sigma}$ is obtained in Eq.~(\ref{eq:svdMSM}) and ${{\mathbf{\bar{d}}}_i}$ is the difference vector, ${\mathbf{u}}_i-{\mathbf{v}}_i$ as shown in Fig.\ref{fig:concept-ds}(a). This attribute reflects the hypervolume of a "gap" between the two subspaces.
\end{dfn} 

\section{Second-order difference subspace}
\label{sec:2ndDS}
We extend the concept of the first-order DS ${\mathcal{D}}$ to the second-order DS ${\mathcal{D}}^{2}$.
\subsection{Definition based on second-order central difference}
Our extension is motivated by the second-order central difference method for solving various differential equations. 
Given three sequential vectors $\mathbf{f}(t-\Delta{t})$, $\mathbf{f}(t)$ and $\mathbf{f}(t+\Delta{t})$ in $n$-dimensional vector space as shown in Fig.\ref{fig:three subspaces}(a), we can write the second-order central difference as follows:
\begin{align}
 \frac{d^2{\mathbf{f}}}{{dt}^2} &\approx
 \bigg( \frac{{\mathbf{f}}(t+\Delta{t})-{\mathbf{f}}(t)}{\Delta{t}}-\frac{{\mathbf{f}}(t)-{\mathbf{f}}(t-\Delta{t})}{\Delta{t}} \bigg)\Big /{\Delta{t}} \nonumber\\
 &= \frac{{\mathbf{f}}(t+\Delta{t})-2{\mathbf{f}}(t)+{\mathbf{f}}(t-\Delta{t})}{\Delta{t}^2} = \frac{2}{{\Delta{t}^2}}\bigg( \frac{{\mathbf{f}}(t+\Delta{t})+{\mathbf{f}}(t-\Delta{t})}{2}-{\mathbf{f}}(t) \bigg ),
\end{align}
where we assume that $\Delta{t}$=1 and the vectors $\mathbf{f}(t-\Delta{t})$, $\mathbf{f}(t)$ and $\mathbf{f}(t+\Delta{t})$ are normalized in length. 

Next, given three sequential subspaces ${\mathcal{S}}_1, {\mathcal{S}}_2$ and ${\mathcal{S}}_3$ in $n$-dimensional vector space as shown in Fig.\ref{fig:three subspaces}(b), we correspond the mean vector, $\mathbf{m}$=$\frac{{\mathbf{f}}(t-1)+{\mathbf{f}}(t+1)}{2}$, to the Karcher mean (principal component subspace), ${\mathcal{M}}({\mathcal{S}}_1, {\mathcal{S}}_3)$, representing the mean subspace of ${\mathcal{S}}_1$ and ${\mathcal{S}}_2$, while ${\mathbf{f}}(t)$ corresponds to ${\mathcal{S}}_2$ as well. 
Then, we correspond the difference vector, $\mathbf{m}-\mathbf{f}(t)$, to the difference subspace between ${\mathcal{M}({\mathcal{S}}_1, {\mathcal{S}}_3)}$ and ${\mathcal{S}}_2$. Remember we have already generalized the difference vector $\mathbf{d}$ between two normalized vectors to the difference subspace $\mathcal{D}$ between two subspaces as described in Fig.\ref{fig:DS}(a). Thus, this correspondence is natural when $\mathbf{f}(t-1)$ and $\mathbf{f}(t+1)$ are close enough such that $||\mathbf{m}||\approx 1.0$. 

Accordingly, we define the second-order difference subspace ${\mathcal{D}}^{2}$ of ${\mathcal{S}}_1, {\mathcal{S}}_2$ and ${\mathcal{S}}_3$ as the difference subspace between ${\mathcal{S}}_2$ and ${\mathcal{M}({\mathcal{S}}_1, {\mathcal{S}}_3)}$.
\begin{dfn}[Second-order difference subspace]
\begin{align}
    {\mathcal{D}}^{2}(\mathcal{S}_1,\mathcal{S}_2, \mathcal{S}_3) = {\mathcal{D}}(S_{2}, \mathcal{M}(\mathcal{S}_1,\mathcal{S}_3)). \nonumber
  \end{align}    
\end{dfn}

\begin{figure}[tb]  
  \begin{minipage}{0.33\linewidth}
  \centering
  \includegraphics[scale=0.37]{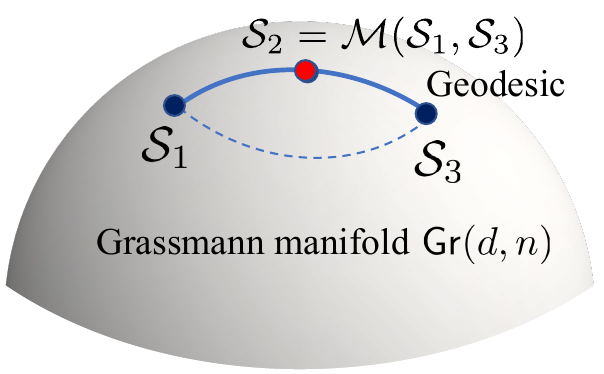}\\
  (a) 
 \end{minipage} 
  \begin{minipage}{0.33\linewidth}
  \centering
  \includegraphics[scale=0.37]{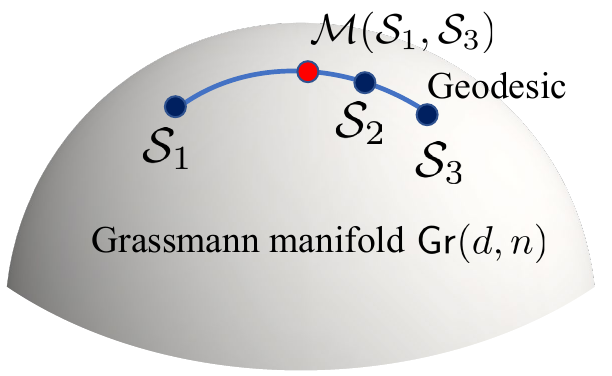}\\
  (b) 
 \end{minipage} 
 \begin{minipage}{0.33\linewidth}
 \centering
 \includegraphics[scale=0.37]{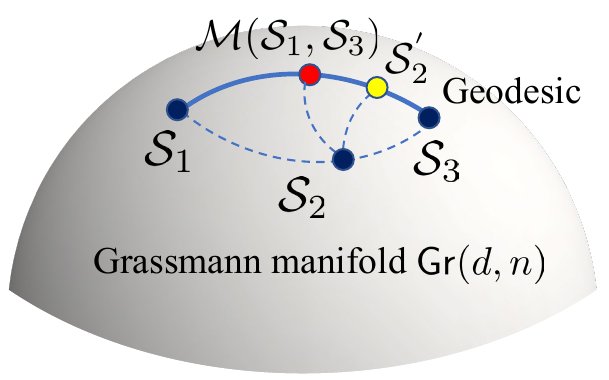}\\
 (c) 
  \end{minipage}
  \caption{Temporal sequential subspaces, $\mathcal{S}_1$, $\mathcal{S}_2$ and $\mathcal{S}_3$ on Grassman manifold $\Gr(d,n)$. (a) Zero second-order DS $\mathcal{D}^2$ in the case that two subspaces ${\mathcal{S}}_2$ and ${\mathcal{M}}$ coincide completely. (b) Second-order DS $\mathcal{D}^2$ in the case that ${\mathcal{S}}_2$ is slightly shifted from ${\mathcal{M}}$ along the geodesic. (c) Second-order DS $\mathcal{D}^2$ in that ${\mathcal{S}}_2$ is not on the geodesic.}
  \label{fig:def_secondDS}
\end{figure}

Further, we define the magnitude $\Mag({\mathcal{D}}^{2})$ of second-order DS 
according to Definition \ref{eq:mag1}.
\begin{dfn}[Magnitude of second-order difference subspace]
\begin{align}
   \Mag({\mathcal{D}}^{2}(\mathcal{S}_1,\mathcal{S}_2, \mathcal{S}_3)) = \Mag({\mathcal{D}}(\mathcal{S}_2,\mathcal{M}(\mathcal{S}_1,\mathcal{S}_3)).
   \label{eq:defmag}
\end{align}
\end{dfn}

\section{Mechanism of second-order DS}
\label{sec:geometry}
\subsection{Representation based on Grassman manifold}
Assuming three subspaces ${\mathcal{S}}_1, {\mathcal{S}}_2$ and ${\mathcal{S}}_3$ of identical dimension $d$, we can understand the geometrical mechanism of the proposed definition more clearly by considering the Grassmann manifold $\Gr(d,n)$, where each $d$-dimensional subspace in $n$-dimensional vector space is represented by a point on $\Gr(d,n)$~\cite{bendokat2024grassmann}.

The geometry of three sequential subspaces in the vector space shown in Fig.\ref{fig:three subspaces}(b) can be visualized as Fig.\ref{fig:def_secondDS}(c) on the Grassmann manifold. A local neighborhood of $n$-dimensional vectors in $n$-dimensional vector space can be identified with Grassmann manifold $\Gr(1,n)$. In this sense, our definition of second-order DS can be regarded as a natural generalization of the second-order central difference on $\Gr(1,n)$ to that on $\Gr(d,n)$.

\subsection{Orthogonal decomposition of magnitude}
{\color{black}
We introduce a geodesic path $l({\mathcal{S}}_i, {\mathcal{S}}_j)$ through ${\mathcal{S}}_i$ and ${\mathcal{S}}_j$, which is the shortest path between points ${\mathcal{S}}_i$ and ${\mathcal{S}}_j$ on $\Gr(d,n)$.
\vspace{1mm}
\begin{lemma}
\label{lemma:sum}
Sum subspace ${\mathcal{W}}({\mathcal{S}}_i, {\mathcal{S}}_j)$ is equal to sum subspace ${\mathcal{W}}(\{{\mathcal{S}}^*\})$ where $\{{\mathcal{S}}^{*}\}$ is a set of all the subspaces on the geodesic ${l}$ through the points ${\mathcal{S}}_i$ and ${\mathcal{S}}_j$ on $\Gr(d,n)$ \cite{Kobayashi_2023_BMVC}. 
The dimensions of  ${\mathcal{S}}_i$, ${\mathcal{S}}_j$ and ${\mathcal{W}}({\mathcal{S}}_i,{\mathcal{S}}_j)$ are $d$, $d$ and $2d$, respectively.

{\noindent}Proof: see Lemma \ref{lemma:sumspace} in the appendix.
\label{lemma:sumspace}
\end{lemma}

The lemma shows that the $2d$-dimensional sum subspace ${\mathcal{W}}({\mathcal{S}}_1, {\mathcal{S}}_3)$ contains any $d$-dimensional subspaces on the geodesic path ${l}({\mathcal{S}}_1, {\mathcal{S}}_3)$. Motivated by this characteristic, we introduce the subspace projection $\omega({\mathcal{S}}_2)$ of ${\mathcal{S}}_2$ onto ${\mathcal{W}}({\mathcal{S}}_1, {\mathcal{S}}_3)$ for obtaining its projected point (subspace) on ${l}({\mathcal{S}}_1, {\mathcal{S}}_3)$. 

\vspace{3mm}
\begin{dfn}[Subspace projection]
\label{def:subspace projection}
The subspace projection $\omega({\mathcal{S}})$ of $d_1$-dimensional subspace $\mathcal{S}$ onto $d_2 (\geq d_1)$-dimensional subspace $\mathcal{W}$ is defined as follows:
\begin{align}
\label{min}
\omega(\mathcal{S}) =\underset{\mathcal{S}^\prime{} \in \mathcal{W}} {\operatorname{argmin}}~\rho(\mathcal{S}, \mathcal{S}^{\prime}),
\end{align}
where $\rho(\mathcal{S}, \mathcal{S}^{\prime})$ indicates the geodesic distance between $\mathcal{S}$ and $\mathcal{S}^{\prime}$.
\end{dfn}

\vspace{2mm}
\begin{lemma}
\label{lemma:subspace projection}
The subspace projection $\omega({\mathcal{S}})$ of $d_1$-dimensional subspace $\mathcal{S}$ onto $d_2(\geq d_1)$-dimensional subspace $\mathcal{W}$ can be calculated by using the singular value decomposition (SVD) as flows:
\begin{align}
{\rm SVD}: {\mathbf{W}}^{\top} {\mathbf{S}} = {\mathbf{U}}{\mathbf{\Sigma}}{\mathbf{V}}^{\top},\\
\omega({\mathcal{S}}): {\mathcal{S}} \rightarrow {\mathcal{S}}^{\prime} \Rightarrow  {\mathbf{S}} \rightarrow {\mathbf{S}}^{\prime}={\mathbf{WU}},
\end{align}
where ${\mathbf{W}} \in \mathbb{R}^{n{\times}d_2}$ and ${\mathbf{S}} \in \mathbb{R}^{n{\times}d_1}$ are the orthonormal basis matrices of $\mathcal{W}$ and $\mathcal{S}$, respectively.
\end{lemma}
For the proof, see \ref{lemma:subspace projection} in the appendix.

\vspace{2mm}
On the basis of the subspace projection of $d$-dimensional ${\mathcal{S}}_2$ onto $2d$-dimensional ${\mathcal{W}}({\mathcal{S}}_1, {\mathcal{S}}_3)$, the magnitude of second-order DS $\Mag({\mathcal{D}}^{2})$ can be represented as the sum of two components: 
$\Mag({\mathcal{D}}(\mathcal{S}_2,\mathcal{M}(\mathcal{S}_1,\mathcal{S}_3)) \simeq \Mag({\mathcal{D}}(\mathcal{S}_2,\mathcal{W}(\mathcal{S}_1,\mathcal{S}_3))\nonumber + \Mag({\mathcal{D}}(\omega(\mathcal{S}_2),\mathcal{M}(\mathcal{S}_1,\mathcal{S}_3))$, where one component is in the direction orthogonal to the geodesic and another is along the geodesic as shown in Fig.\ref{fig:def_secondDS}(c). 

It may seem strange that the subspace projection can link the two different operations in Euclidean and non-Euclidean spaces (Grassman manifold). We will provide proof of the validity of the subspace projection and the orthogonal decomposition in the future, although a numerical simulation supports them.
} %<== end of color{red}

\if 0
\begin{lemma}
\label{lemma:pcs}
Principal component subspace ${\mathcal{M}}({\mathcal{S}}_1, {\mathcal{S}}_3)$ is the middle point between two points 
${\mathcal{S}}_1$ and ${\mathcal{S}}_3$ along the geodesic ${l}({\mathcal{S}}_1, {\mathcal{S}}_3)$ on $\Gr(d,n)$ \cite{Kobayashi_2023_BMVC}.
\end{lemma}
\fi
\if 0
\begin{lemma}
\begin{align}
   \Mag({\mathcal{D}}(\mathcal{S}_2,\mathcal{M}(\mathcal{S}_1,\mathcal{S}_3)) &\simeq \Mag({\mathcal{D}}(\mathcal{S}_2,\mathcal{W}(\mathcal{S}_1,\mathcal{S}_3))\nonumber + \Mag({\mathcal{D}}(\mathcal{S}_2^{'},\mathcal{M}(\mathcal{S}_1,\mathcal{S}_3)),
   \label{eq:defmag}
\end{align}
{\color{black}
where ${\mathcal{S}}^{\prime}_{2}$ is the projection of $\mathcal{S}_2$ onto $\mathcal{W}({\mathcal{S}}_1$, ${\mathcal{S}}_3)$, thus being ${\mathcal{S}}^{\prime}_{2}$ $\in \mathcal{W}({\mathcal{S}}_1$, ${\mathcal{S}}_3)$,   
}
%where {\color{black}${\mathcal{S}}^{\prime}_{2}$ is on the pseudo geodesic; ${\mathcal{S}}^{\prime}_{2}$ $\in \mathcal{W}({\mathcal{S}}_1$, ${\mathcal{S}}_3)$ in vector space $\mathbb{R}^n$. ${\mathcal{S}}^{\prime}_{2}$ is a subspace generated by applying the Gram–Schmidt orthonormalization to the projected basis of ${\mathcal{S}}_{2}$ onto the sum subspace $\mathcal{W}({\mathcal{S}}_1$, ${\mathcal{S}}_3)$ (Lemma \ref{lemma:sum}). }
For the details, see the appendix.
\end{lemma}
\fi

We analyze how the second-order DS works based on the decomposition in the following three cases.
Fig.\ref{fig:def_secondDS}(a) shows the case that ${\mathcal{S}}_2$ and ${\mathcal{M}}$ coincide completely. In this case, the magnitude of the second-order DS is zero, representing no acceleration. Fig.\ref{fig:def_secondDS}(b) shows the case that ${\mathcal{S}}_2$ shifts from ${\mathcal{M}}$ along the geodesic on $\Gr(d,n)$; ${\mathcal{M}}$, ${\mathcal{S}}_2$ $\in$ ${\mathcal{W}}({\mathcal{S}}_1, {\mathcal{S}}_3)$ in vector space $\mathbb{R}^n$. In this case, the magnitude of the second-order DS corresponds to acceleration only along the geodesic, exhibiting $\Mag({\mathcal{D}}(\mathcal{S}_2,\mathcal{W}(\mathcal{S}_1,\mathcal{S}_3))>0$ while {\color{black} $\Mag({\mathcal{D}}(\omega(\mathcal{S}_2),\mathcal{W}(\mathcal{S}_1,\mathcal{S}_3)) = 0$.}
Fig.\ref{fig:def_secondDS}(c) shows the case that ${\mathcal{S}}_2$ departs from the geodesic; ${\mathcal{M}}$ $\in$ ${\mathcal{W}}({\mathcal{S}}_1, {\mathcal{S}}_3)$ and ${\mathcal{S}}_2$ $\notin$ ${\mathcal{W}}({\mathcal{S}}_1, {\mathcal{S}}_3)$ in $\mathbb{R}^n$, where the magnitude of second-order DS consists of accelerations in two directions. 

{\color{black}
We can also introduce the tangent space $T_M$ of $\Gr(d,n)$ with origin at ${\mathcal{M}}$ to understand the mechanism of the orthogonal decomposition. $T_M$ is useful to analyze the Grassman manifold since it is a vector space with ${\mathcal{M}}$ as its origin \cite{absil2004riemannian}. Fig.~\ref{fig:grassman_def}(a) shows the points ${\mathcal{\hat{S}}}_1$, ${\mathcal{\hat{S}}}_2$, ${\mathcal{\hat{S}}}_3$ and ${\mathcal{\hat{S}}}^{'}_2$ in the logarithm mapping~\cite{bendokat2024grassmann} of ${\mathcal{S}}_1$, ${\mathcal{S}}_2$, ${\mathcal{S}}_3$ and ${\mathcal{S}}^{'}_2$ onto the tangent space $T_M$, respectively. In Fig.~\ref{fig:grassman_def}(b), we can also see that the distance between $\mathcal{\hat{S}}_2$ and $\mathcal{M}$ can be decomposed into two orthogonal components: the distance between ${\mathcal{\hat{S}}}_2$ and ${\mathcal{\hat{S}}}_{2}^{\prime}$ and that between ${\mathcal{\hat{S}}}_{2}^{\prime}$ and $\mathcal{M}$, 
where ${\mathcal{\hat{S}}}^{\prime}_2$ indicates the orthogonal projection of ${\mathcal{\hat{S}}}_2$ onto the geodesic line in the tangent space. 
}

%From Lemma \ref{lemma:sumspace}, $\mathcal{W}$ represents the geodesic through two points, ${\mathcal{S}}_1$ and ${\mathcal{S}}_3$.
%When three subspaces, $\mathcal{S}_1, \mathcal{S}_2$ and $\mathcal{S}_3$ are close to each other, ${\mathcal{D}}({\mathcal{S}}_2, {\mathcal{W}}(\mathcal{S}_1,\mathcal{S}_3)))$ and ${\mathcal{D}}({\mathcal{S}}_2^{'},\mathcal{M}(\mathcal{S}_1,\mathcal{S}_3))$ are almost orthogonal. 
%Thus, this orthogonal decomposition is explained by the Pythagorean theorem.

\begin{figure}[tb]
 \begin{minipage}{0.48\linewidth}
  \centering
  \includegraphics[scale=0.38]{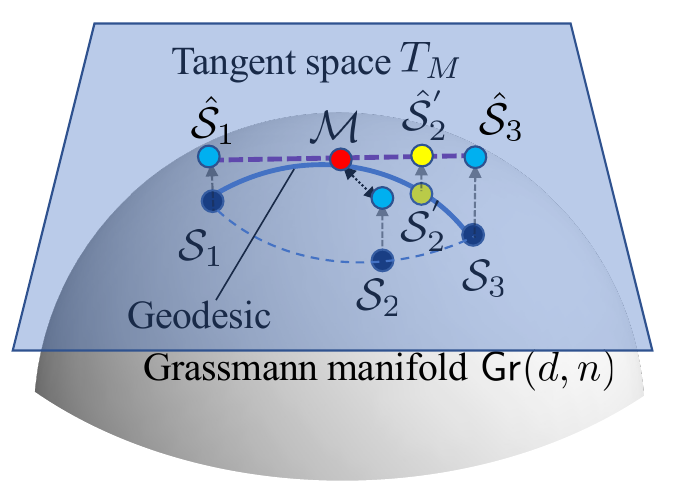}\\
  (a) Logarithmic mapping of the related subspaces on the tangent space $T_M$ of $\Gr(d,n)$.
  \end{minipage}  
  \begin{minipage}{0.48\linewidth}
  \centering
  \includegraphics[scale=0.38]{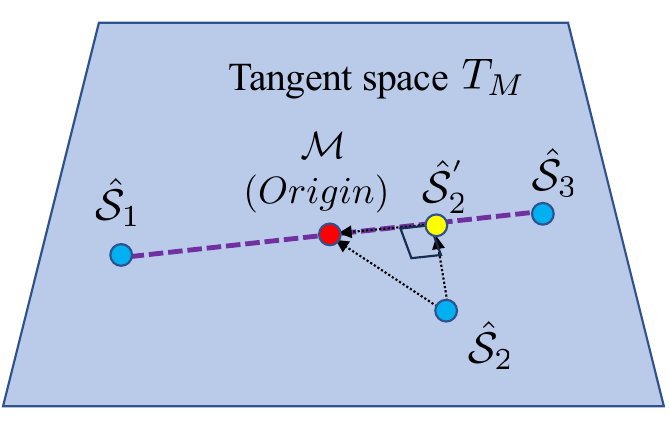}\\
  (b) Orthogonal decomposition on the tangent space $T_M$.
 \end{minipage}\\
 \caption{Geometrical relationship of the related subspaces on the $\Gr(d,n)$ and the tangent space  $T_M$ of $\Gr(d,n)$.} 
   \label{fig:grassman_def}
\end{figure}

%==============================================
\section{Numerical experiments}
\label{sec:numexps}
In this section, we provide a brief numerical study on two applications to demonstrate the validity and naturalness of the definitions of our first/second-order DSs. The first application is a temporal analysis of deforming 3D shapes using a three-dimensional shape subspace, representing the changing pose of a human. This experimental setting is comparatively simple as there is no intersection between the sequential shape subspaces and the dimension of subspaces is only three. The second application is a time sequence analysis of biometric signals using signal subspaces generated by applying the singular spectrum analysis (SSA) to an iput signal. This setting is more complicated than the first one in that each signal subspace has a higher dimension than three, and there is an intersection between sequential signal subspaces.

\subsection{Temporal analysis of deforming 3D shape}
In this experiment, we calculate the magnitudes of the first/second-order DSs from the temporal sequence of shape subspaces. In this setting, a 3D pose of a human at time $t$ is represented by a three-dimensional subspace ${\mathcal{S}}_t$ in high-dimensional vector space\cite{shapesp,shapesp2010}. A shape subspace can be generated from a set of 3D dot points extracted from a walking or jumping human. The significant merit of the shape subspace-based representation is that it is invariant to any affine transformation against the whole set of 3D points. In other words, the shape subspace is invariant to viewpoints.

Given $p$ 3D points of ${\{x_i,y_i,z_i\}}_{i=1}^p$ of a 3D object (whole human body) at frame time $t$, we first generate the following matrix $\mathbf{V}$:
\begin{equation}
{\mathbf{V}}=
\begin{pmatrix}
x_1 - m_x& y_1 -m_y & z_1 -m_z\\
x_2 - m_x& y_2 -m_y & z_2 -m_z\\
\vdots & \vdots & \vdots\\
x_p - m_x & y_p -m_y & z_p -m_z \nonumber
\end{pmatrix},
\end{equation}
where $(m_x,m_y,m_z)$ indicates the center of gravity of all $p$ points.
We then apply the Gram-Schmidt orthogonalization to the column vectors of $\mathbf{V}$. The orthogonalized column vectors ${\mathbf{\hat{V}}}$ is the basis of shape subspace ${\mathcal{S}}_t$. Finally, a task of temporal analysis of 3D pose is converted to measuring the variations of first/second-order DSs between sequential shape subspaces in $p$-dimensional space.

 \begin{figure}[tb]
  \begin{minipage}[t]{0.48\linewidth}
  \centering
   \includegraphics[scale=0.2]{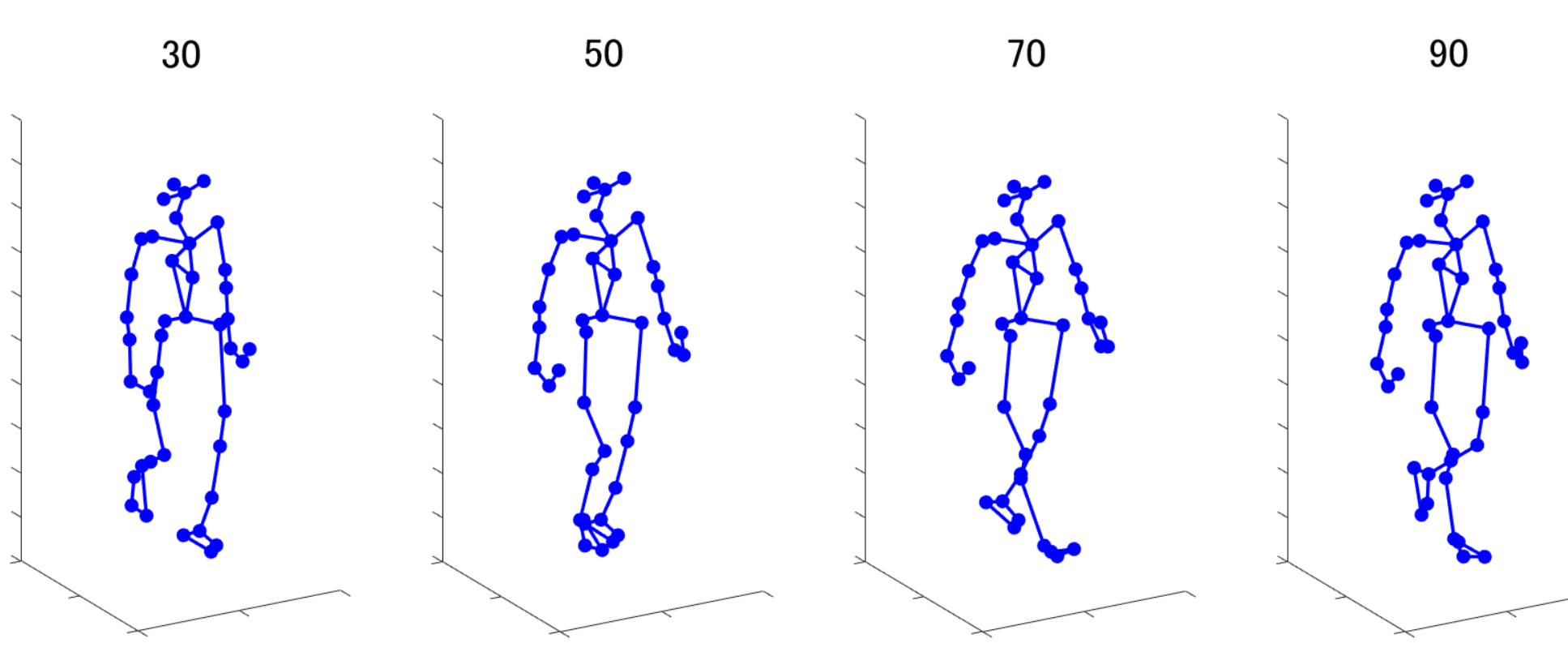}
   (a) walking
  \end{minipage}
  \hspace{7mm}
   \begin{minipage}[t]{0.48\linewidth}
  \centering
   \includegraphics[scale=0.2]{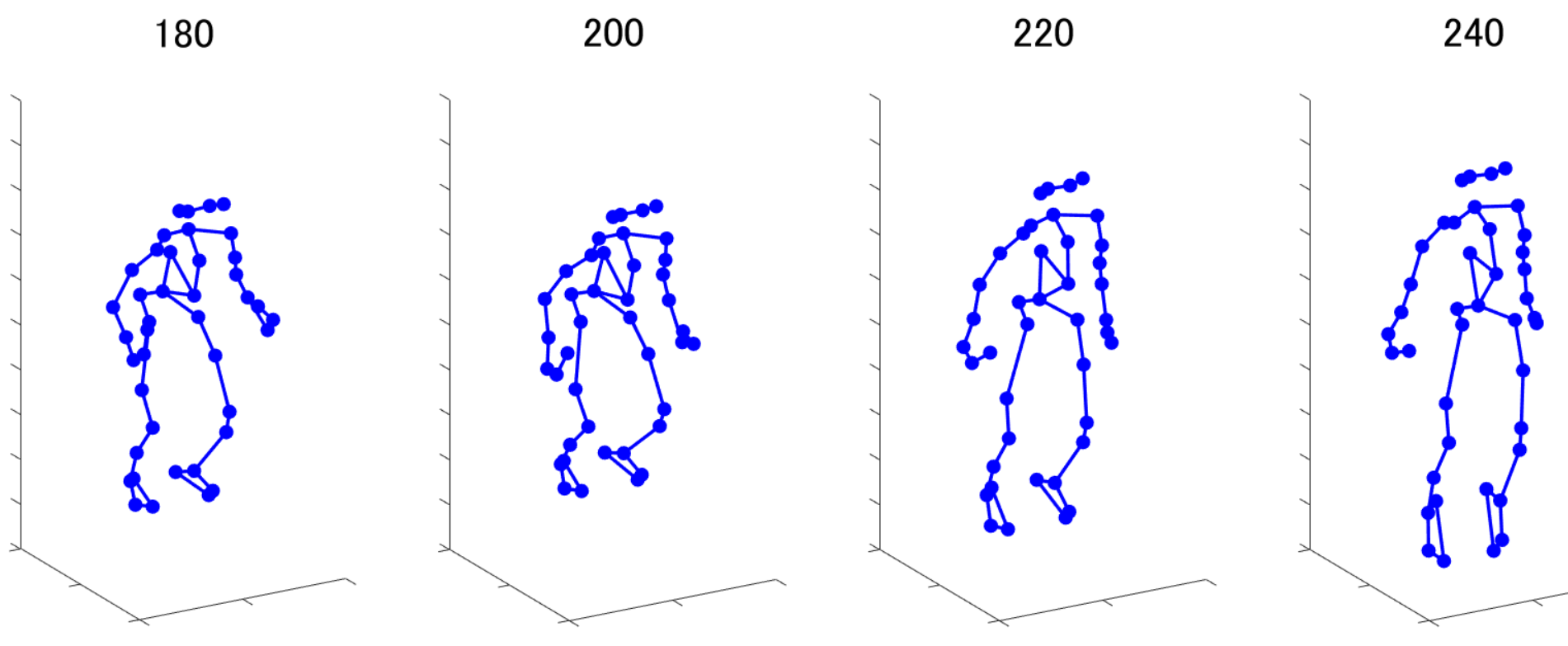}
   (b) jumping, turning around
  \end{minipage}  
  \caption{Sequential sets of 3D dots of walking and jumping data. The number of dots is 41. Each number indicates the frame number.}
   \label{fig:input3Ddots}
\end{figure}

\begin{figure}[tb]
\centering
 \begin{minipage}[t]{0.24\linewidth}
  \centering
  \includegraphics[scale=0.21]{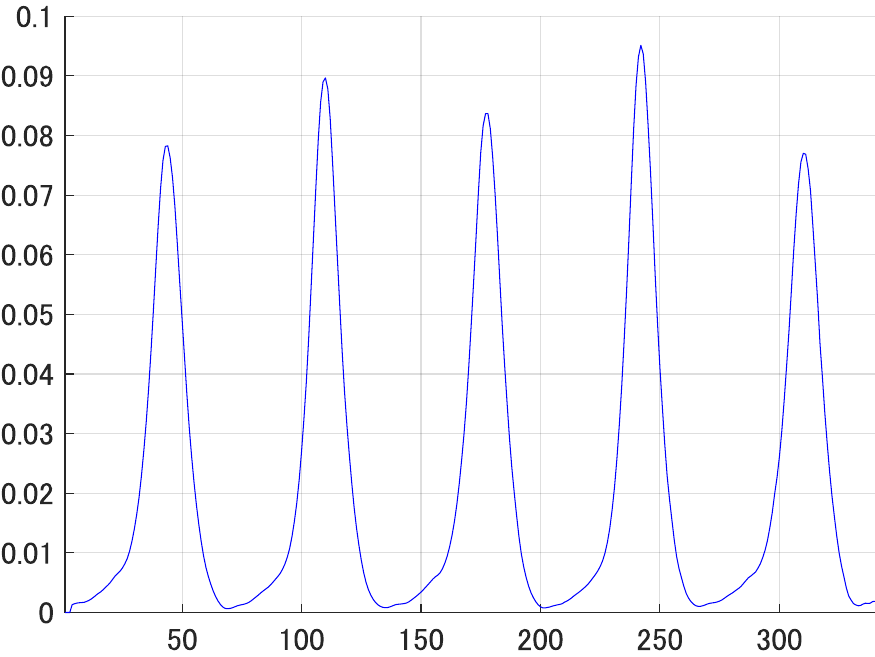}\\
  (a) First-order DS, ${\mathcal{D}}(\mathcal{S}_1,\mathcal{S}_3)$
 \end{minipage}
\begin{minipage}[t]{0.24\linewidth}
  \centering
  \includegraphics[scale=0.21]{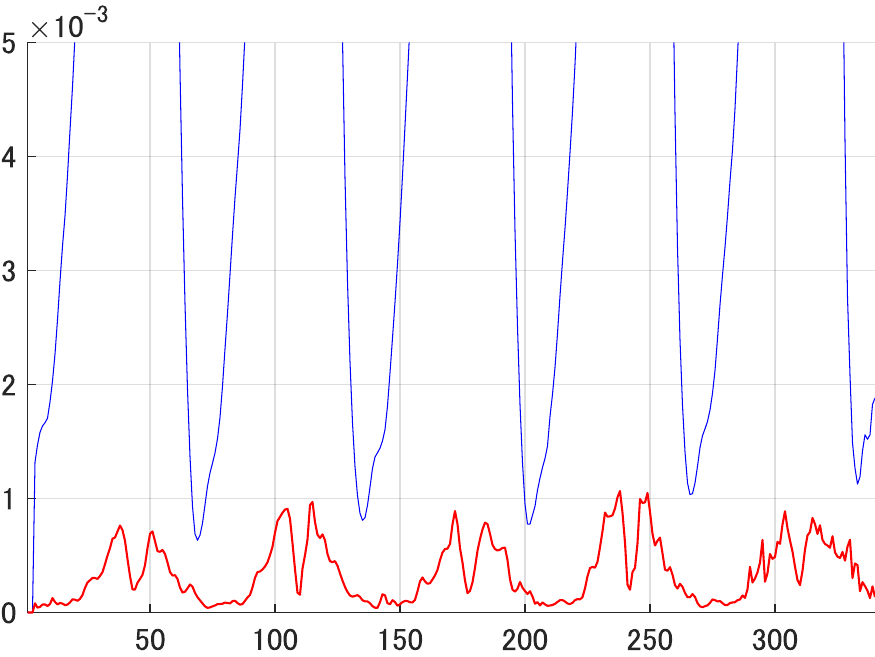}\\
  (b) Second-order DS, ${\mathcal{D}}(S_{2}, \mathcal{M}(\mathcal{S}_1,\mathcal{S}_3))$(red) and ${\mathcal{D}}(\mathcal{S}_1,\mathcal{S}_3)$ (blue).
 \end{minipage}
\begin{minipage}[t]{0.24\linewidth}
  \centering
  \includegraphics[scale=0.21]{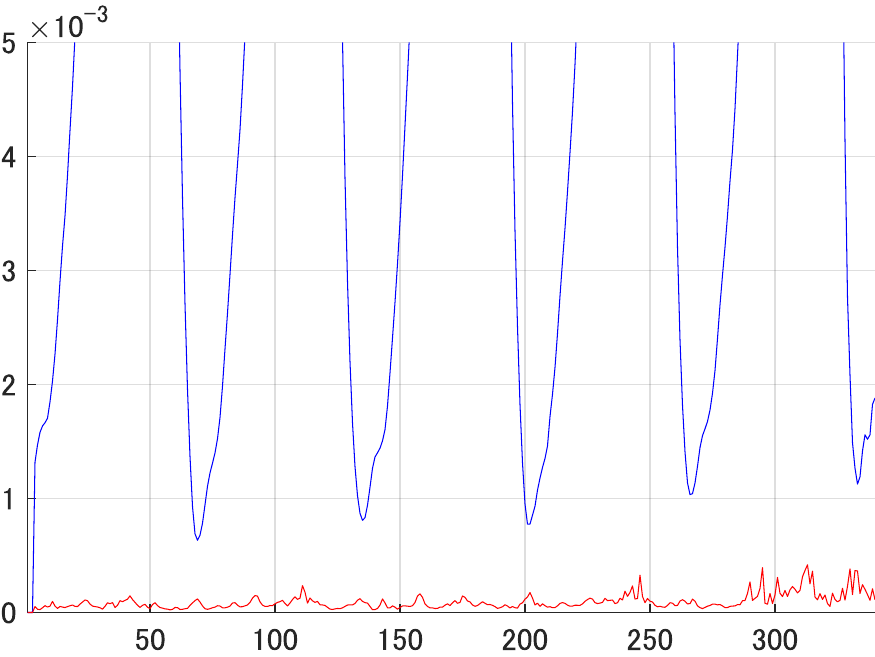}\\
  (c) The orthogonal component to the geodesic, ${\mathcal{D}}({\mathcal{S}}_{2}, \mathcal{W}(\mathcal{S}_1,\mathcal{S}_3))$
 \end{minipage}
 \begin{minipage}[t]{0.24\linewidth}
  \centering
  \includegraphics[scale=0.21]{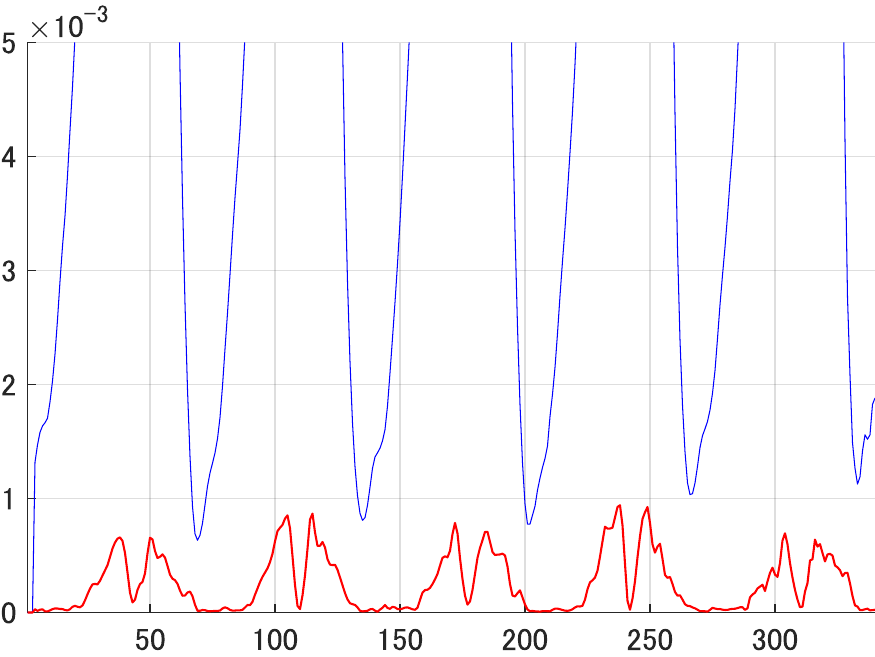}\\
  (d) The component along the geodesic, ${\mathcal{D}}({\mathcal{S}}^{\prime}_{2}, \mathcal{M}(\mathcal{S}_1,\mathcal{S}_3))$
 \end{minipage}
 \caption{The magnitudes of first/second-order DSs on 3D dots data of walking.}
 \label{fig:output3Ddots-walking}
 \end{figure}

\begin{figure}[tb]
  \centering
  \includegraphics[scale=0.18]{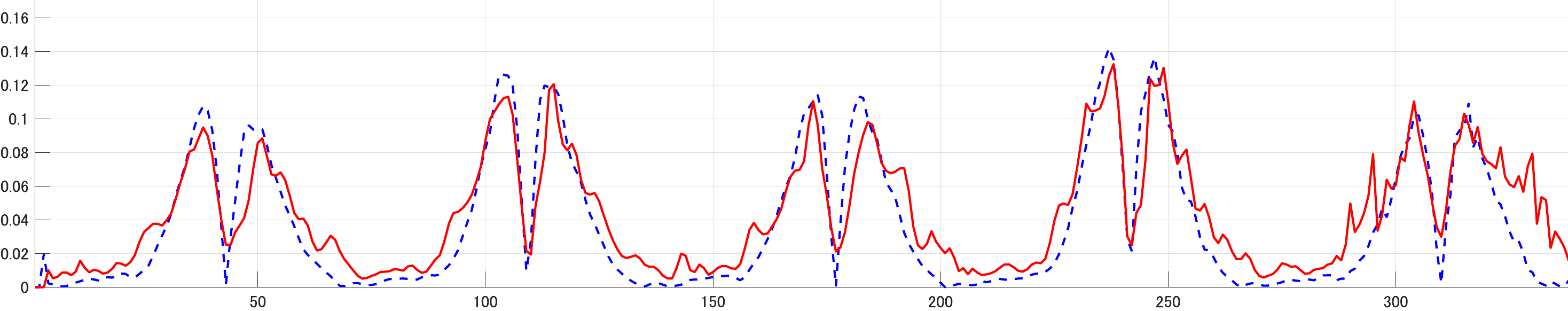}
  \caption{{\color{black}Comparison of the second-order DS output with the absolute value of derivative of the first-order DS's output.} The dashed line indicates the derivative of the first-order DS. Both output sequence vectors are normalized.}
   \label{fig:comp3Ddots-comp}
\end{figure}

%=== jumping, turn around ===%
\begin{figure}[tb]
\centering
 \begin{minipage}[t]{0.4\columnwidth}
  \centering
  \includegraphics[scale=0.25]{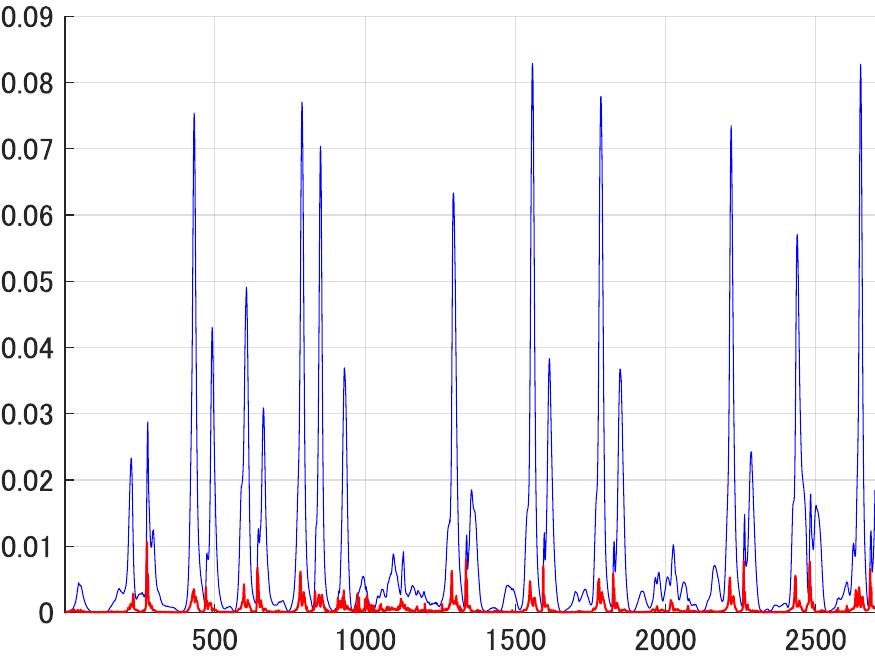}\\
  (a) First-order DS (blue) and second-order DS(red). 
\end{minipage}
\if 0
\begin{minipage}[t]{0.58\columnwidth}
  \centering
  \includegraphics[scale=0.25]{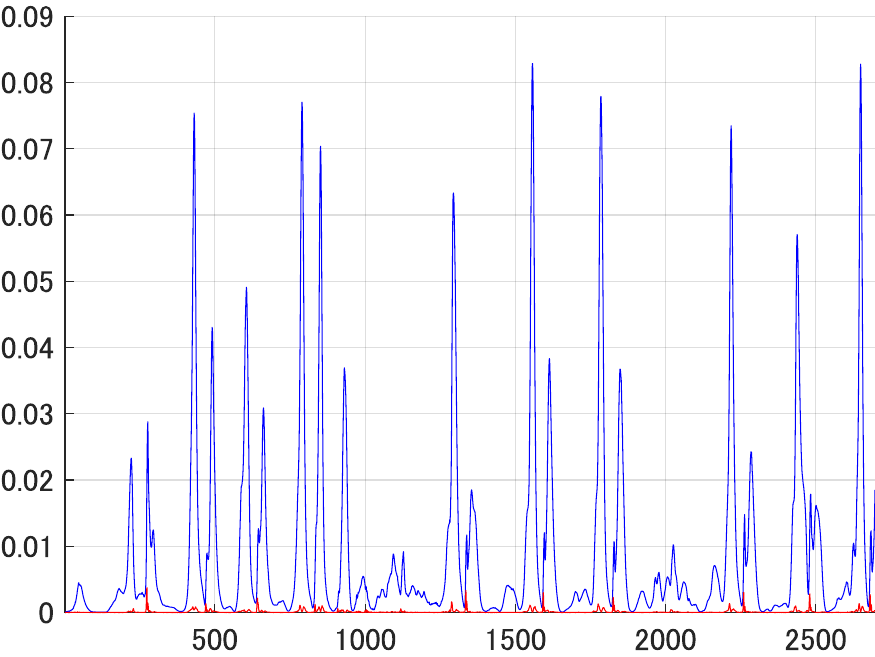}\\
  (b) Enlarged drawing of (a)
\end{minipage}
\fi
\begin{minipage}[t]{0.32\columnwidth}
  \centering
  \includegraphics[scale=0.3]{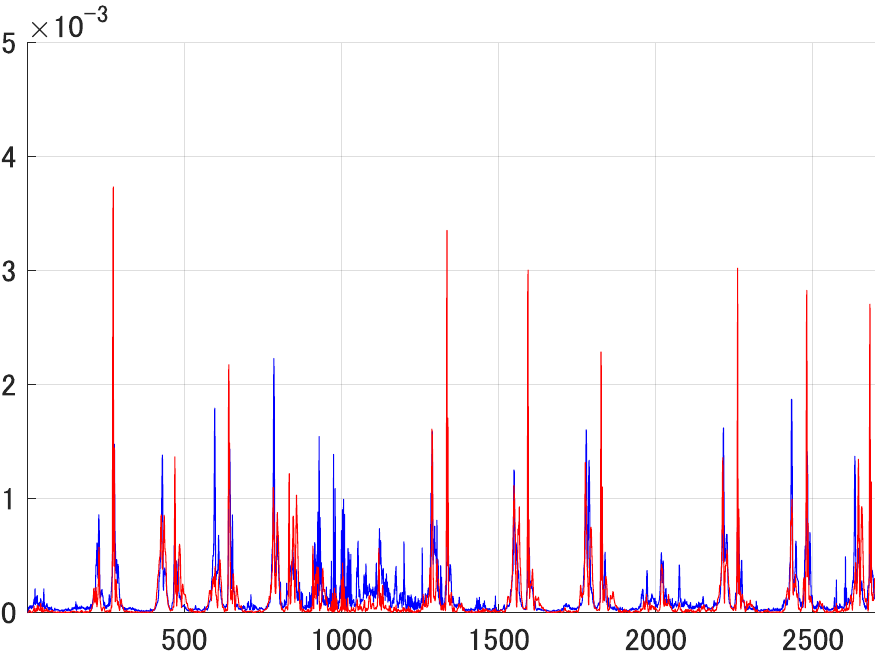}\\
  (b) The component (blue) orthogonal to the geodesic and one (red) along the geodesic. 
 \end{minipage}
 \caption{The magnitudes of first/second-order DSs on 3D dots data of jumping, turn around. }
 \label{fig:output3Ddots-jumping}
 \end{figure}

\subsection{Experimental setting}
We used two sequential sets of 3D dot data from the CMU Graphics Lab Motion Capture Database \cite{cmudatabase}. One represents "walking action". Another represents "jumping, turning around". Figs.\ref{fig:input3Ddots}(a) and (b) show example sets of 3D dot data of walking and jumping. The number of 3D dot data is 41. Thus, three-dimensional shape subspaces are in 41-dimensional vector space. The total number of walking and jumping frames are 343 and 2701, respectively. We used 3D shape subspaces that extracted every four frames, as the change between successive frames is very small.

\subsection{Experimental results}
Figure \ref{fig:output3Ddots-walking}(a) shows the magnitude of the first-order DS for walking data. We can see that it can capture the periodic variation in velocity during walking. Fig.\ref{fig:output3Ddots-walking}(b) shows the comparison between the magnitudes of the first-order DS (blue line) and the second-order DS (red line). The acceleration matches well with the change in velocity. Figs.\ref{fig:output3Ddots-walking}(c) and (d) show the two orthogonal components of the magnitude of the second-order DS. The component of ${\mathcal{D}}^{\prime}({\mathcal{S}}^{\prime}_{2}, \mathcal{M}(\mathcal{S}_1,\mathcal{S}_3))$ is dominant in the walking data, while the component of ${\mathcal{D}}({\mathcal{S}}_{2}, \mathcal{W}(\mathcal{S}_1,\mathcal{S}_3))$ is small. This difference implies that the primary acceleration occurs along the geodesic. 
It is known that a movement along a geodesic can be smooth, satisfying the energy minimum principle. In this sense, we may consider walking as a smooth action.

{\color{black}
To further confirm the naturalness of the definition of our second-order DS, we compared 
the absolute value of the derivative of the first-order DS's output with the second-order DS's. Both sequence vectors were normalized. We can see that they have almost the same variation pattern shown in Fig.\ref{fig:comp3Ddots-comp}. In fact, the normalized correlation coefficient is a high value of 0.948.}

Figure \ref{fig:output3Ddots-jumping}(a) shows the magnitude of the first-order DS for jumping data. We can see that it captures the periodic variation in the velocity of jumping. Fig.\ref{fig:output3Ddots-jumping}(b) shows the magnitudes of the first-order DS (blue line) and the second-order DS (red line). Figs.\ref{fig:output3Ddots-jumping}(c) and (d) show two orthogonal components of the second-order DS's magnitude. The two orthogonal components, ${\mathcal{D}}(\omega({\mathcal{S}}_2), \mathcal{M}(\mathcal{S}_1,\mathcal{S}_3))$ and ${\mathcal{D}}(S_{2}, \mathcal{W}(\mathcal{S}_1,\mathcal{S}_3))$ contribute equally to the whole acceleration unlike the result of walking. 
This result suggests that jumping, unlike walking, may not be regarded as a smooth action that satisfies the energy minimum principle.

\subsection{Time series analysis of biometric signal}
We demonstrate the validity of first/second-order DSs in analyzing one-channel signal \cite{mssa,gssa,tgssa} based on the framework of the singular spectrum analysis (SSA)\cite{SSAbook}. 
SSA-based method for signal analysis relies on a low dimensional subspace, called signal subspace, generated in one of the steps of SSA as shown in Fig.\ref{fig:basicidea}. The main advantage of using signal subspace is that it can represent the essential temporal structure of signal data compactly, hence largely reducing the computational cost. 
%Moreover, the signal subspace can be stably generated using the singular value decomposition even if the learning data is insufficient, unlike the probability function.

The process flow of the SSA-based method consists of the following four steps, as shown in Fig.\ref{fig:basicidea}. First, the entire time series data is divided into two parts: past and present time series. Second, two signal subspaces are generated by applying the SSA to the past and present time-series data. Next, the magnitude of the first /second-order DS between the present and past subspaces is measured as the degree of anomaly change. Finally, a specific anomaly change is detected when the anomaly score is larger than a given threshold value. 
%This experimental setting is more complicated than the first experiment, as two sequential signal subspaces intersect. This is because these sequential signal subspaces are generated from two similar, partly overlapping time series data as shown in Fig.\ref{fig:basicidea}.

\noindent{\bf{Generation of signal subspace}}\\
\noindent{The signal subspace} ${\mathcal{S}}_t$ corresponding to a time-series data $h(t)$ is generated by analyzing the trajectory matrix calculated from the time-series data in the process of the SSA as shown in Fig.~\ref{fig:basicidea}. Given one-dimensional time series data $h(t)$, the corresponding trajectory matrix ${\mathbf{H}}_t \in \mathbb{R}^{w{\times}M}$ is defined as follows:
\begin{equation}
\label{eq:hankel}
   {\mathbf{H}}_t=\left[\begin{array}{cccc}
 h(t-w-M+2) & \cdots & h(t-w+1) \\
\vdots & \ddots & \vdots \\
 h(t-M+1) & \cdots & h(t)
 \end{array}\right],
\end{equation}
where $w$ is the width of a sliding window and \textit{M} is the number of the sliding windows as shown in Fig.~\ref{fig:basicidea}. 

To obtain the principal components of $h(t)$, SSA solves the eigenvalue problem of ${\mathbf{H}}_{t} {\mathbf{H}}_{t}^{\top} {\mathbf{\Phi}}={\mathbf{\Phi \Sigma}}$, where $\mathbf{\Phi}$ is the matrix arranging eigenvectors in columns and $\mathbf\Sigma$ is the matrix containing eigenvalues, $\lambda_1, \dots, \lambda_w$, in the diagonal elements.
The $d_s$-dimensional signal subspace ${\mathcal{S}}_{t}$ of the input time series data $h(t)$ is spanned by the eigenvectors $\{{\boldsymbol{\phi}}_i\}_{i=1}^{d_s}$ corresponding to the $d_s$ largest eigenvalues in $w$-dimensional vector space.

\begin{figure}[tb]
\begin{minipage}[t]{1.0\columnwidth}   
\centering
\includegraphics[scale=0.5]{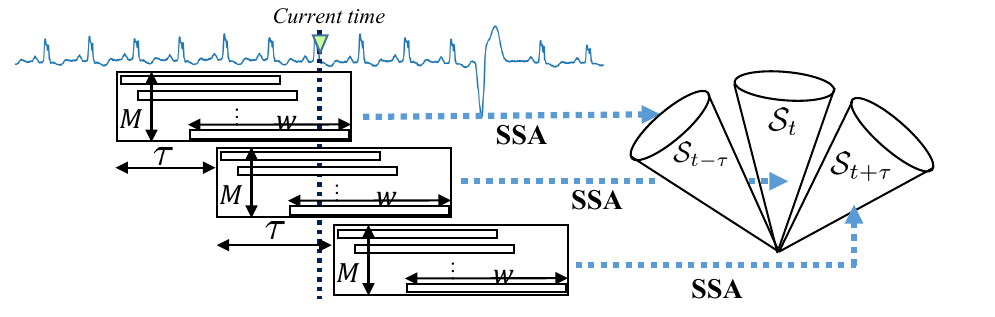}
%(a) First-order DS
%\end{minipage}
%\begin{minipage}[t]{0.48\columnwidth}   
%\centering
%\includegraphics[scale=0.5]{fig/timeseriesDD.pdf}\\
%(b) Second-order DS
\end{minipage}
\caption{A framework for analyzing signal using SSA. It measures the temporal variations of the first-order DS between $\mathcal{S}_{t-\tau}$ and $\mathcal{S}_{t+\tau}$ and the second-order DS of $\mathcal{S}_{t-\tau}$, $\mathcal{S}_t$ and $\mathcal{S}_{t+\tau}$.} 
\label{fig:basicidea}
\end{figure}

%The essence of our method is to monitor the temporal variation of DS between the past and present signal subspaces in a high dimensional vector space. 
%To this end, we consider the variation of DS in two terms: the direction and magnitude. For the direction, we measure the dissimilarity of the present difference subspace $ \mathcal{D}_{in}$ with non-anomalous difference subspaces ${\mathcal{D}_{N}}$ as an index, where the non-anomalous difference subspace is generated from normal time-series data without any change in the learning phase. In this paper, we generate the non-anomalous difference subspace from early time-series data, assuming that there is no anomaly during that term. For the magnitude, we use the sum of the cosines of multiple canonical angles between the past and present signal subspaces as an index. Finally, we use the product of the two indices as the change degree of our method.

\begin{figure}[tb]
 \begin{minipage}[t]{0.3\linewidth}
  \centering
  \includegraphics[scale=0.2]{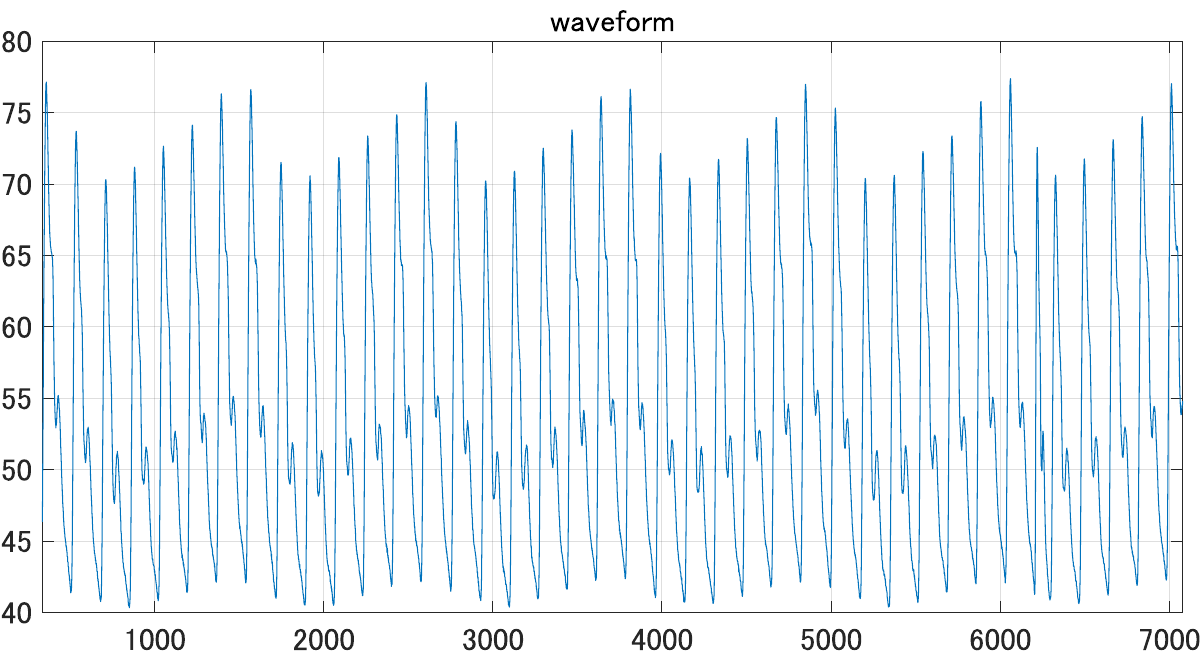}\\
  (a) Input sequential signal
  \end{minipage} \hspace{3mm}
 \begin{minipage}[t]{0.3\linewidth}
  \centering
  \includegraphics[scale=0.2]{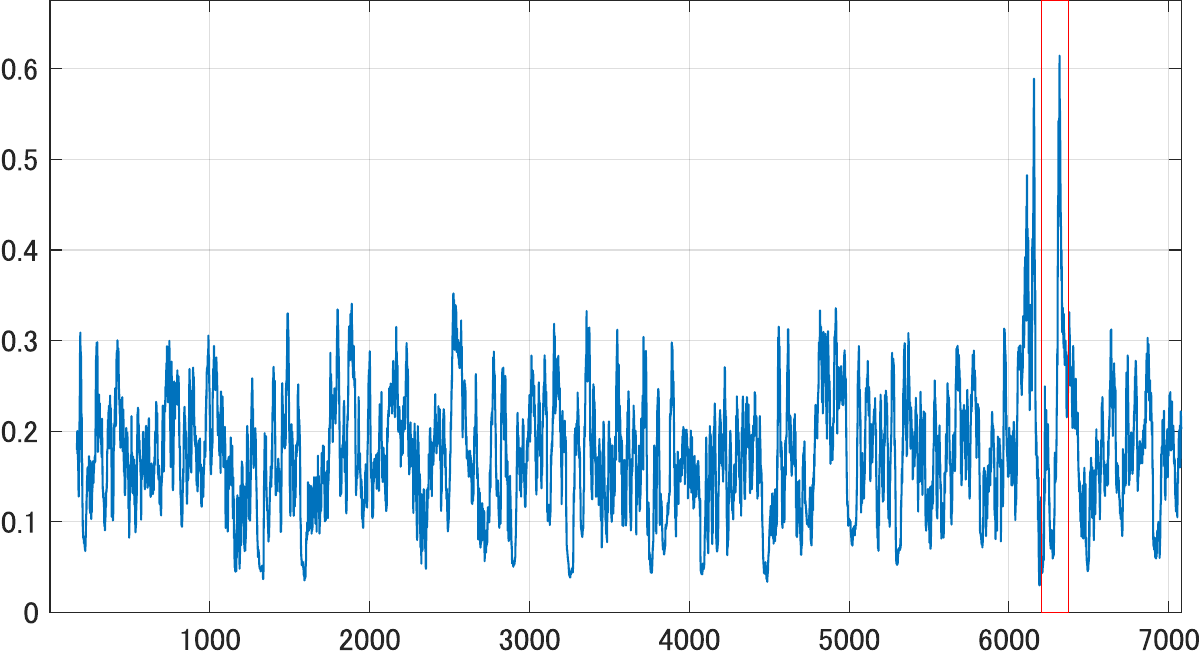}\\
  (b) $\Mag({\mathcal{D}}(\mathcal{S}_{t-\tau},\mathcal{S}_{t+\tau}))$
\end{minipage} \hspace{3mm}
 \begin{minipage}[t]{0.33\linewidth}
 \centering
  \includegraphics[scale=0.2]{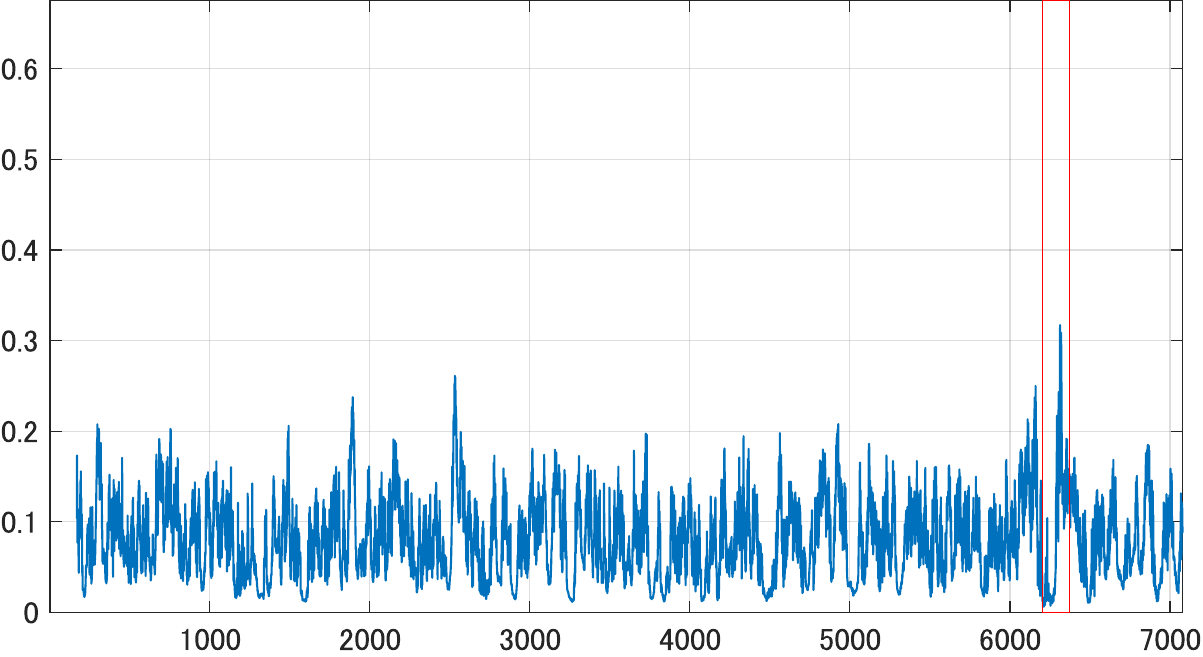}\\
  (c) $\Mag({\mathcal{D}}(\mathcal{S}_{t},\mathcal{M}(\mathcal{S}_{t-\tau},\mathcal{S}_{t+\tau}))$
\end{minipage}
\caption{The input signal and its magnitudes of first/second-order DSs.}
\label{fig:result1}
\end{figure}

\subsection{Experimental setting}
For this experiment, we used the following data from the UCR Time Series Data Mining Archive \cite{Eamonn2005-ea}: 141\_UCR\_Anomaly\_InternalBleeding5\_4000\_6200\_6370, which is shown in Fig.\ref{fig:result1}(a). The parameters, M and w, were set to 220 and 100, respectively. The dimension of signal subspaces was 40. Time lag $\tau$ was set to 16 and $\delta$ in Sec.\ref{sec:general case} was set to 1e-4. 

\subsection{Experimental results}
Figs.\ref{fig:result1}(b) and (c) show the magnitudes of the first/second-order DSs. We can see that both the first and second-order DSs respond to the abnormal term (red rectangle) as expected. Figs.\ref{fig:result1}(d) and (e) show $\Mag({\mathcal{D}}(\mathcal{S}_t,\mathcal{W}(\mathcal{S}_{t-\tau},\mathcal{S}_{t+\tau}))$ and $\Mag({\mathcal{D}}({\mathcal{\hat{S}}}_t,\mathcal{M}(\mathcal{S}_{t-\tau},\mathcal{S}_{t+\tau}))$, respectively. The main acceleration occurs along the geodesic, while the acceleration in the direction orthogonal to the geodesic is small, as in the first experiment using walking data. Therefore, we can mention that this biometrical signal (cardiac electrogram) changes smoothly, satisfying the energy minimum principle.

%\section{Discussion}
%We discuss the characteristic of each definition of second-order difference subspace.

\if 0
\section{Limitations}
\label{sec:limitations}
A first-order DS has a symmetrical structure: ${\mathcal{D}}({\mathcal{S}}_1, {\mathcal{S}}_3)$=${\mathcal{D}}({\mathcal{S}}_3, {\mathcal{S}}_1)$. Thus, we cannot distinguish them in our framework. Similarly, ${\mathcal{D}}({\mathcal{S}}_1, {\mathcal{S}}_2, {\mathcal{S}}_3)$ is equal to ${\mathcal{D}}({\mathcal{S}}_3, {\mathcal{S}}_2, {\mathcal{S}}_1)$, as we defined the second-order DS based on the first-order DS. We should consider this property carefully when applying our second-order DS to applications. 
%
We can calculate a second-order DS in most cases. However, when focusing on the property in terms of derivative, we need to assume that the related subspaces are comparatively close to each other, as the definition of the second-order DS is motivated by the idea of the central difference method under the condition that the associated vectors are close to each other.
\fi

\section{Conclusions}
\label{sec:conclusions}
We proposed the novel concept of second-order difference subspace as an extension of first-order difference subspace.
As preliminary to its definition, we defined the first-order difference subspace in a more general setting in which there is an intersection subspace between two subspaces with different dimensions. Given three sequential subspaces, we defined second-order difference subspace as the difference subspace between the principal component subspace between the first and third subspaces and the middle subspace. The numerical experiments on 3D shape analysis using 3D shape subspace and time series analysis of biomedical signals clearly demonstrated the validity and naturalness of the definition of the second-order DS.

%%%%%%%%%%%%%%%%%%%%%%%%%%%%%%%%%%%%%%%%%%%%%%%%%%%%%%%%%%%%

\section*{Acknowledgments}
We would like to thank Eamonn Keogh and all the other people who have contributed to the UCR Time Series Data Mining Archive. We would like to thank the CMU Graphics Lab Motion Capture Database. The data used in this research was obtained from mocap.cs.cmu.edu. The database was created with funding from NSF EIA-0196217.

\newpage
\bibliographystyle{unsrt} %ieeetr
\bibliography{neurips_2024}

\newpage
\appendix

\section{Appendix}
%-------------------New proof------------------
\begin{lemma}
\label{lemma:new orthonormal basis}
The orthonormal basis ${\mathbf{\Phi}} \in {\mathbb{R}}^{{n}\times{d_1}}$ of ${\mathcal{S}_1}$ and $\mathbf{\Psi} \in {\mathbb{R}}^{{n}\times{d_2}}$ of ${\mathcal{S}_2}$ ($d_1 \leq d_2$) can be orthogonally transformed to the following orthonormal bases ${\mathbf{\Phi}}^{*} \in {\mathbb{R}}^{{n}\times{d_1}}$ and ${\mathbf{\Psi}}^{*} \in {\mathbb{R}}^{{n}\times{d_2}}$: 
%Let ${\mathcal{S}_1}, {\mathcal{S}_2} \in \mathbb{R}^n$ be subspaces of dimension $d_1$ and $d_2$ respectively, where $d_1 \leq d_2$. ${\mathcal{S}_1}$ and ${\mathcal{S}_2}$ have an $r$-dimensional intersection subspace ${\mathcal{I}}$. The following $\mathbf{\Phi}^{*}$ and $\mathbf{\Psi}^{*}$ are possible orthonormal basis for ${\mathcal{S}_1}$ and ${\mathcal{S}_2}$.
\begin{align}
\mathbf{\Phi}^{*} &=[ {\boldsymbol{\gamma}}_1, \dots, {\boldsymbol{\gamma}}_r, {\mathbf{u}}_{r+1},  \dots, {\mathbf{u}}_{d_1} ], \\
\mathbf{\Psi}^{*} &= [ {\boldsymbol{\gamma}}_1, \dots, {\boldsymbol{\gamma}}_r,  {\mathbf{v}}_{r+1},  \dots, {\mathbf{v}}_{d_1}, {\boldsymbol{\beta}}_{d_1+1}, \dots, {\boldsymbol{\beta}}_{d_2} ],
\end{align} 
where $\{{\boldsymbol{\gamma}}_i\}$ represents the orthonormal basis of an  intersection subspace $\mathcal{I}$ between ${\mathcal{S}_1}$ and ${\mathcal{S}_2}$, $\{{\mathbf{u}}_i\}$ and $\{{\mathbf{v}}_i\}$ are a pair of the canonical vectors forming the $i$th nonzero canonical angle $\theta_i$ between 
${\mathcal{S}_1}$ and ${\mathcal{S}_2}$.
$\{{\boldsymbol{\beta}}_{d_1+1}, \dots, {\boldsymbol{\beta}}_{d_2}\}$ are orthogonal to the remaining orthonormal basis,  $\{{\boldsymbol{\gamma}}_i\}$, $\{{\mathbf{u}}_i\}$ and $\{{\mathbf{v}}_i\}$.
%where $\{{\boldsymbol{\gamma}}_i\}$ represents the orthonormal basis of intersection subspace $\mathcal{I}$, $\{{\mathbf{u}}_i\}$ and $\{{\mathbf{v}}_i\}$ are a pair of the canonical vectors forming the $i$th canonical angle $\theta_i$, and $\{{\boldsymbol{\beta}}_i\}$ are orothogonal to the other orthonormal basis, $\{{\boldsymbol{\gamma}}_i\}$, $\{{\mathbf{u}}_i\}$ and $\{{\mathbf{v}}_i\}$.
%$\{{\boldsymbol{\beta}}_i\}$ are not included in either of ${\mathcal{M}}$ and ${\mathcal{D}}$. 
\end{lemma}

\begin{proof}
We can orthogonally decompose $\mathcal{S}_1$  into two subspaces: the intersection $\mathcal{I}$ with $\mathcal{S}_2$  and non-intersection $\mathcal{A}_1$ as $\mathcal{S}_1 = \mathcal{I} \oplus \mathcal{A}_1$, where  $\mathcal{I} = \Span (\boldsymbol{\Gamma})$  and ${\boldsymbol{\Gamma}} = [ {\boldsymbol{\gamma}}_1, \dots, {\boldsymbol{\gamma}}_r ]$. 
Thus, an orthonormal basis of $\mathcal{S}_1$ can be written in the form of a block matrix $\mathbf{\Phi} = [\mathbf{\Gamma} \quad \mathbf{A}_1] $ where $\mathbf{A}_1 \in {\mathbb{R}}^{n\times(d_1-r)}$ is the orthonormal basis of subspace ${\mathcal{A}}_1$.

Analogously, we can orthogonally decompose  $\mathcal{S}_2$  into three subspaces
as $\mathcal{S}_2 = \mathcal{I} \oplus \mathcal{A}_2 \oplus \mathcal{A}_3$ where  $\mathcal{A}_1$ and $\mathcal{A}_2$ have the same dimension of $d_1-r$. Hence,  an orthonormal basis of $\mathcal{S}_2$ can be written in the form of a block matrix $\mathbf{\Psi} = [\mathbf{\Gamma} \quad \mathbf{A}_2 \quad \mathbf{A}_3] $ where $\mathbf{A}_2 \in {\mathbb{R}}^{n\times(d_1-r)}$ and $\mathbf{A}_3 \in {\mathbb{R}}^{n\times(d_2-d_1)}$.

We can focus on only the relationship among $\mathcal{A}_1, \mathcal{A}_2$ and $\mathcal{A}_3$, since these subspaces are orthogonal to $\mathcal{I}$. 
Then, we apply SVD to ${\mathbf{A}}_1^{\top} [\mathbf{A}_2 \quad \mathbf{A}_3]$:
    \begin{align}
    \label{svd3}
        {\mathbf{A}}_1^{\top}  [\mathbf{A}_2 \quad \mathbf{A}_3] =
        \mathbf{U}
        \begin{bmatrix}
            \mathbf{\Sigma} & \bf{0}
        \end{bmatrix}
        \begin{bmatrix}
            {\mathbf{V}} & {\mathbf{V}}_{\perp}
        \end{bmatrix}^{\top},
    \end{align}
where ${\mathbf{U}} \in {\mathbb{R}}^{(d_1-r)\times(d_1-r)}$,
${\mathbf{\Sigma}} \in {\mathbb{R}}^{(d_1-r)\times(d_1-r)}$,
${\mathbf{0}} \in {\mathbb{R}}^{(d_1-r)\times(d_2-d_1)}$,
${\mathbf{V}} \in {\mathbb{R}}^{(d_2-r)\times(d_1-r)}$  and  
${\mathbf{V}}_{\perp} \in {\mathbb{R}}^{(d_2-r)\times(d_2-d_1)}$.
Note that ${\mathbf{U}}^{\top} {\mathbf{U}} ={\mathbf{V}}^{\top} {\mathbf{V}} ={\mathbf{V}}_{\perp}^{\top} {\mathbf{V}}_{\perp} =\mathbf{I}$ and ${\mathbf{V}}^{\top} {\mathbf{V}}_{\perp} =\mathbf{0}$ .

Then, we generate new orthonormal basis as follows:
    \begin{align}
    \mathbf{B}_1&=\mathbf{A}_1 \mathbf{U},\\
    \mathbf{B}_2&=[\mathbf{A}_2 \quad \mathbf{A}_3]\mathbf{V},\\
    \mathbf{Z}&=[\mathbf{A}_2 \quad \mathbf{A}_3] {\mathbf{V}}_{\perp},   
    \end{align}
where  ${\mathbf{B}}_1 \in {\mathbb{R}}^{n\times(d_1-r)}$ and ${\mathbf{B}}_2  \in {\mathbb{R}}^{n\times(d_1-r)}$  are  the pairs of the canonical vectors forming the canonical angles $\{\theta_i\}_{i=r+1}^{d_1}$, as   
    ${\mathbf{B}}_1^{\top} {\mathbf{B}}_2=\mathbf{\Sigma}=\diag(\mathbf{\Theta})$.
    The size of ${\mathbf{Z}}$ is ${n}$ by $d_2-d_1$.

It follows from ${\mathbf{\Gamma}}^{\top} \mathbf{A}_1 = {\mathbf{\Gamma}}^{\top} \mathbf{A}_2 = {\mathbf{\Gamma}}^{\top} \mathbf{A}_3=\bf{0}$ that ${\mathbf{\Gamma}}^{\top} {\mathbf{B}}_1 
={\mathbf{\Gamma}}^{\top} {\mathbf{B}}_2 
={\mathbf{\Gamma}}^{\top} {\mathbf{Z}}_1 =\bf{0}$.

Further, we show that ${\mathbf{B}}_2^{\top} \mathbf{Z} = \bf{0}$ as follows:   
\begin{align} 
    {\mathbf{B}}_2^{\top} \mathbf{Z} &= {\mathbf{V}}^{\top} 
    \begin{bmatrix}
    {\mathbf{A}}_2^{\top} \\
    {\mathbf{A}}_3^{\top} 
    \end{bmatrix}
    \begin{bmatrix}
    {\mathbf{A}}_2 &   {\mathbf{A}_3} 
    \end{bmatrix}
    {\mathbf{V}}_{\perp},\\     
    &=   {\mathbf{V}}^{\top} 
    \begin{bmatrix}
    {\mathbf{I}}  & {\bf{0}}\\
    {\bf{0}} & {\mathbf{I}} 
    \end{bmatrix}
    {\mathbf{V}}_{\perp},\\
    &= {\mathbf{V}}^{\top} {\mathbf{V}}_{\perp},\\
    &=\bf{0}.  
\end{align}
% ${\mathbf{U}}^{\top} {\mathbf{A}}_1^{\top}  {\mathbf{A}_3} {\mathbf{V}}_{\perp}=\bf{0}$ from Eq.(\ref{svd3}).

Consequently, we obtain the orthonormal basis  ${\mathbf{\Phi}}^{*}$ and 
${\mathbf{\Psi}}^{*}$ as follows:
\begin{align}
\mathbf{\Phi}^{*} &=[ {\boldsymbol{\Gamma}} \quad \mathbf{B}_1], \\
\mathbf{\Psi}^{*} &= [ {\boldsymbol{\Gamma}} \quad \mathbf{B}_2 \quad \mathbf{Z}],
\end{align} 
where ${\mathbf{B}}_1 = [ {\mathbf{u}}_{r+1},  \dots, {\mathbf{u}}_{d_1}]$, 
${\mathbf{B}}_2 = [ {\mathbf{v}}_{r+1},  \dots, {\mathbf{v}}_{d_1}]$ and 
${\mathbf{Z}} = [ {\boldsymbol{\beta}}_{d_1+1},  \dots, {\boldsymbol{\beta}}_{d_2}]$. 

Next, we show that ${\mathbf{B}}_1^{\top} \mathbf{Z} = \bf{0}$ as follows:

From Eq.(\ref{svd3}), we obtain
\begin{align}    
    \mathbf{U}^{\top}{\mathbf{A}}_1^{\top}  [\mathbf{A}_2 \quad \mathbf{A}_3]            
     \begin{bmatrix}
            {\mathbf{V}} & {\mathbf{V}}_{\perp}
     \end{bmatrix}=
     \begin{bmatrix}
            \mathbf{\Sigma} & \bf{0}
      \end{bmatrix} . 
\end{align}     
Hence,
\begin{align}    
    \mathbf{U}^{\top}{\mathbf{A}}_1^{\top}  [\mathbf{A}_2 \quad \mathbf{A}_3]            
   {\mathbf{V}}_{\perp} =  \bf{0}.
\end{align}  
$ {\mathbf{B}}_1^{\top} \mathbf{Z}$ is written as follows:
\begin{align} 
    {\mathbf{B}}_1^{\top} \mathbf{Z} = {\mathbf{U}}^{\top} 
    {\mathbf{A}}_1^{\top} 
    \begin{bmatrix}
    {\mathbf{A}}_2 &   {\mathbf{A}_3} 
    \end{bmatrix}
    {\mathbf{V}}_{\perp}.\\     
\end{align}
Therefore, 
\begin{align}    
{\mathbf{B}}_1^{\top} \mathbf{Z}= \mathbf{U}^{\top}{\mathbf{A}}_1^{\top}  [\mathbf{A}_2 \quad \mathbf{A}_3] {\mathbf{V}}_{\perp} =  \bf{0}. 
\end{align}

Finally, we conclude that ${\mathbf{\Gamma}} {\mathbf{Z}}^{\top}
={\mathbf{B}}_1 {\mathbf{Z}}^{\top}
={\mathbf{B}}_2 {\mathbf{Z}}^{\top}
=\bf{0}$, 
implying that the three subspaces, corresponding to ${\mathbf{B}}_1$, ${\mathbf{B}}_2$ and ${\mathbf{Z}}$, are orthogonal to each other.
\end{proof}

\vspace{3mm}
\begin{lemma}
\label{lemma:eigenvalues of G}
Given two subspaces $\mathcal{S}_1$ and  $\mathcal{S}_2$ such that they have an intersection. Let ${\mathbf{\Phi}}^{*}$ and ${\mathbf{\Psi}}^{*}$ be the orthonormal basis for each subspace, as defined in the previous lemma, and let 
$\mathbf{P}_1$=${\mathbf{\Phi}}{{\mathbf{\Phi}}}^{\top}$=${{\mathbf{\Phi}}^{*}}{{\mathbf{\Phi}}^{*}}^{\top}$ and 
$\mathbf{P}_2$=${\mathbf{\Psi}}{{\mathbf{\Psi}}}^{\top}$=${{\mathbf{\Psi}}^{*}}{{\mathbf{\Psi}}^{*}}^{\top}$ 
be the orthogonal projection matrix onto each subspace. 
Let the sum matrix $\mathbf{G}=\mathbf{P}_1+\mathbf{P}_2$. 
The principal component subspace $\mathcal{M}$ is spanned by the eigenvectors of $\mathbf{G}$ corresponding to the $d_1$ larger eigenvalues  than 1.
The difference subspace $\mathcal{D}$ is spanned by the eigenvectors of $\mathbf{G}$ corresponding to $d_1 - r$ non-negative eigenvalues smaller than 1.    
 \end{lemma}
 
\begin{proof}
We again consider the non-intersection subspaces $\mathcal{A}_1$ and $\mathcal{A}_2$ and calculate the sum of their projection matrices as:
    \begin{equation}
        \mathbf{J} = {\mathbf{B}_1} {\mathbf{B}}_1^\top + {\mathbf{B}_2} {\mathbf{B}}_2^\top,
    \end{equation}
    \begin{equation}
        \mathbf{J}\mathbf{H}  = \mathbf{H}\mathbf{S}.
    \end{equation}
By the definitions of $\mathbf{B}_1$ and $\mathbf{B}_2$, we can derive the projection matrices of $\mathcal{S}_1$ and  $\mathcal{S}_2$, and then $\mathbf{G}$:
    \begin{align}
        {\mathbf{\Phi}^{*}} {\mathbf{\Phi}^{*}}^\top &= \mathbf{\Gamma} \mathbf{\Gamma}^\top + {\mathbf{B}_1} {\mathbf{B}}_1^{\top}, \\
         {\mathbf{\Psi}^{*}} {\mathbf{\Psi}^{*}}^{\top} &= \mathbf{\Gamma} \mathbf{\Gamma}^\top + {\mathbf{B}_2} {\mathbf{B}}_2^\top +  {\mathbf{Z}} {\mathbf{Z}}^\top,\\
        {\mathbf{G}} &= 2 \mathbf{\Gamma} \mathbf{\Gamma}^\top + \mathbf{J} +  {\mathbf{Z}} {\mathbf{Z}}^\top,
    \end{align}
where $\mathbf{\Gamma} \mathbf{\Gamma}^{\top}$, $\mathbf{Z}\mathbf{Z}^{\top}$ and $\mathbf{J}$ are orthogonal to each other. Thus, the SVD of  $\mathbf{G}$ is of the following form:
    \begin{equation}
        \mathbf{G} = 
        \begin{bmatrix}
            \mathbf{\Gamma} & \mathbf{Z} & \mathbf{H}
        \end{bmatrix}
        \begin{bmatrix}
            2 \mathbf{I} & &\\
            & \mathbf{I} &\\
             & & \mathbf{S}
        \end{bmatrix}
        \begin{bmatrix}
            \mathbf{\Gamma} & \mathbf{Z} & \mathbf{H}
        \end{bmatrix}^\top.
    \end{equation}
The intersection $\mathcal{I}$ is spanned by the leading eigenvectors with value $2$. The subspace $\mathcal{Z}$ is spanned by the eigenvectors with value $1$. 
The remainder are then the same eigenvalues of $\mathbf{J}$. Since $\mathbf{J}$ is the sum of the two non-intersected subspaces, they follow the simple case definition of principal/ difference subspaces, such that the principal component subspace $\mathcal{M}$ is spanned by the eigenvectors corresponding to the $d_1$ larger  eigenvalues than 1  and the difference subspace $\mathcal{D}$ is spanned by the $d_1 - r$ non-negative eigenvalues smaller than 1.    
\end{proof}

\begin{lemma}
Sum subspace ${\mathcal{W}}({\mathcal{S}}_1, {\mathcal{S}}_2)$ is equal to ${\mathcal{W}}(\{{\mathcal{S}}^*\})$ where $\{{\mathcal{S}}^{*}\}$ is a set of all the subspaces on the geodesic ${l}$ 
between two points, ${\mathcal{S}}_1$ and ${\mathcal{S}}_2$ on $\Gr(d,n)$.
\label{lemma:sumspace}
\end{lemma}

\begin{proof}
  The projection matrix $\mathbf{G}$ onto  ${\mathcal{W}}(\{{\mathcal{S}}^*\})$ can be written as the integral of all projection matrices of subspaces ${\mathcal{S}}^*$ along the geodesic:
  \begin{equation}
    \mathbf{G} = \int_{0}^{1} \mathbf{\hat{l}}(t) \mathbf{\hat{l}}(t)^\top  dt.
  \end{equation}
In addition, the projection matrix $\mathbf{H}$ onto ${\mathcal{W}}({\mathcal{S}}_1, {\mathcal{S}}_2)$ can be written simply as
\begin{equation}
  \mathbf{H} = \mathbf{\Phi}\mathbf{\Phi}^\top + \mathbf{\Psi}\mathbf{\Psi}^\top.
\end{equation}
  
As shown in~\cite{Kobayashi_2023_BMVC}, the projections matrices can be expressed as the following eigendecompositions:
\begin{align}
   \mathbf{G} = [\mathbf{M}, \mathbf{D}]
  \left[
    \arraycolsep=0.5mm
  \begin{array}{cc}
    \mathbf{\Lambda}_+\\
    &\mathbf{\Lambda}_{\!-}
  \end{array}\right]
  [\mathbf{M}, \mathbf{D}]^\top,
  \label{eq:eigdecompG}\\
  \mathbf{H} = [\mathbf{M}, \mathbf{D}]
  \left[
    \arraycolsep=0.5mm
  \begin{array}{cc}
   \mathbf{S}_+\\
    &\mathbf{S}_{\!-}
  \end{array}\right]
  [\mathbf{M}, \mathbf{D}]^\top,
  \label{eq:eigdecompH}
\end{align}
where $\mathbf{M},\mathbf{D}$ are orthogonal bases for the principal and difference and  $\mathbf{\Lambda}_{\pm}, \mathbf{S}_{\pm}$ eigenvalue matrices, written as
\begin{align}
  \mathbf{M} &\!=\! \mathsf{colnorm}(\mathbf{\Phi}\mathbf{U}+\mathbf{\Psi}\mathbf{V})\in\mathbb{R}^{n\times d},\\
  \mathbf{D} &\!=\! \mathsf{colnorm}(\mathbf{\Phi}\mathbf{U}-\mathbf{\Psi}\mathbf{V})\in\mathbb{R}^{n\times d},\\
  \mathbf{\Lambda}_{\pm}&\! =\!  \diag(\{1\pm\sinc\,\theta_k\}_{k=1}^d), \\
 \mathbf{S}_{\pm}&\! =\! \diag(\{1\pm\cos\theta_k\}_{k=1}^d). %\in\Re^{d\times d},\ 
  \label{eq:eigval}
\end{align}

Since their eigenvectors are exactly the same, it must be that their span is the same, proving our claim.
\end{proof}
%---------------------------------------------------------

\begin{lemma}[The orthonormal basis for ${\mathcal{D}}$ and ${\mathcal{M}}$]
\label{lemma:orthonormal basis of D and M}
Give two subspaces $d_1$-dimensional $\mathcal{S}_1$ and $\mathcal{S}_2$ such that there is no intersection between them. 
Let ${\mathbf{\Phi}}$ and ${\mathbf{\Psi}}$ be the orthonormal basis of $\mathcal{S}_1$ and $\mathcal{S}_2$, respectively. 
The svd of ${\mathbf{\Phi}}^{\top}{\mathbf{\Psi}}$ is written as ${\mathbf{U}}{\mathbf{\Sigma}}{\mathbf{V}}^{\top}$.
The orthonormal basis ${\mathbf{D}}$ and ${\mathbf{M}}$ of difference subspace ${\mathcal{D}}$ and principal component subspace ${\mathcal{M}}$ between the two subspaces $\mathcal{S}_1$ and $\mathcal{S}_2$ are written in matrix form as follows:
\begin{align}
{\mathbf{D}} = ({\mathbf{{\Phi}U-{\Psi}V}})\{2({\mathbf{I}}-{\mathbf{\Sigma)}}\}^{-\frac{1}{2}},\\
{\mathbf{M}} = ({\mathbf{{\Phi}U+{\Psi}V}})\{2({\mathbf{I}}+{\mathbf{\Sigma)}}\}^{-\frac{1}{2}}.
 \label{eq:DS}
\end{align}

\begin{proof}
Each difference vector between a pair of canonical vectors can be written as follows:
\begin{align}
[{\mathbf{u}}_1-{\mathbf{v}}_1, {\mathbf{u}}_2-{\mathbf{v}}_2,\dots, {\mathbf{u}}_{d_1}-{\mathbf{v}}_{d_1}]
&=[\bar{\mathbf{d}}_1, \bar{\mathbf{d}}_2, \cdots, \bar{\mathbf{d}}_{d_1}],\\
&={\mathbf{{\Phi}U-{\Psi}V}} \in {\mathbb{R}}^{n{\times}{d_1}}.
\end{align}
Similarly, each principal component vector between a pair of canonical vectors can be written as follows:
\begin{align}
[{\mathbf{u}}_1+{\mathbf{v}}_1, {\mathbf{u}}_2+{\mathbf{v}}_2,\dots, {\mathbf{u}}_{d_1}+{\mathbf{v}}_{d_1}],
&=[\bar{\mathbf{m}}_1, \bar{\mathbf{m}}_2, \cdots, \bar{\mathbf{m}}_{d_1}],\\
&={\mathbf{{\Phi}U+{\Psi}V}} \in {\mathbb{R}}^{n{\times}{d_1}}.
\end{align}
Using the cosine theorem, we obtain ${\|{\mathbf{\bar{d}}}_i\|}_2^2$ and ${\|{\mathbf{\bar{m}}}_i\|}_2^2$ as follows:
\begin{align}
{\|{\mathbf{\bar{d}}}_i\|}_2^2 &= {\|{\mathbf{u}}_i\|}_2^2+{\|{\mathbf{v}}_i\|}_2^2 - 2 {\|{\mathbf{u}}_i\|}_2 {\|{\mathbf{v}}_i\|}_2 \cos{\theta_i},\\
&= 1 + 1 -2\cos{\theta_i},\\
&= 2(1 -\cos{\theta_i}).\\
{\|{\mathbf{\bar{m}}}_i\|}_2^2 &= {\|{\mathbf{u}}_i\|}_2^2+{\|{\mathbf{v}}_i\|}_2^2 - 2 {\|{\mathbf{u}}_i\|}_2 {\|{\mathbf{v}}_i\|}_2 \cos({\pi-\theta_i}),\\
&= 1 + 1 +2\cos{\theta_i},\\
&= 2(1 +\cos{\theta_i}).
\end{align}

Hence,
\begin{align}
\diag({\|{\mathbf{\bar{d}}}_1\|}_2, {\|{\mathbf{\bar{d}}}_2\|}_2, \dots, {\|{\mathbf{\bar{d}}}_{d_1}\|}_2)
={(2({\mathbf{I}}_{d_1{\times}d_1}-\cos(\mathbf\Theta))}^{\frac{1}{2}},
\end{align}
\begin{align}
\diag({\|{\mathbf{\bar{m}}}_1\|}_2, {\|{\mathbf{\bar{m}}}_2\|}_2, \dots, {\|{\mathbf{\bar{m}}}_{d_1}\|}_2)
={(2({\mathbf{I}}_{d_1{\times}d_1}+\cos(\mathbf\Theta))}^{\frac{1}{2}},
\end{align}
where $\mathbf\Theta=\diag(\theta_1, \theta_2, \dots, \theta_{d_1})$ and $\cos (\mathbf\Theta)=\mathbf{\Sigma}$.

According to the definitons of ${\mathcal{D}}$ and ${\mathcal{M}}$, we conclude that 
\begin{align}
{\mathbf{D}} &=[\bar{\mathbf{d}}_1, \cdots, \bar{\mathbf{d}}_{d_1}]
{\diag({\|{\mathbf{\bar{d}}}_1\|}_2, \dots, {\|{\mathbf{\bar{d}}}_{d_1}\|}_2)}^{-1}
=({\mathbf{{\Phi}U-{\Psi}V}})\{2({\mathbf{I}}-{\mathbf{\Sigma)}}\}^{-\frac{1}{2}}, \\
{\mathbf{M}} &=[\bar{\mathbf{m}}_1, \cdots, \bar{\mathbf{m}}_{d_1}]
{\diag({\|{\mathbf{\bar{m}}}_1\|}_2, \dots, {\|{\mathbf{\bar{m}}}_{d_1}\|}_2)}^{-1}
=({\mathbf{{\Phi}U+\Psi}V})\{2({\mathbf{I}}+{\mathbf{\Sigma)}}\}^{-\frac{1}{2}}.
\end{align}

\end{proof}
\end{lemma}

%----------------------------------------------

\vspace{3mm}
\begin{dfn}[Subspace projection]
\label{def:subspace projection}
The subspace projection $\omega({\mathcal{S}})$ of $d_1$-dimensional subspace $\mathcal{S}$ onto $d_2 (\geq d_1)$-dimensional subspace $\mathcal{W}$ is defined as follows:
\begin{align}
\label{min}
\omega(\mathcal{S}) =\underset{\mathcal{S}^\prime{} \in \mathcal{W}} {\operatorname{argmin}}~\rho(\mathcal{S}, \mathcal{S}^{\prime}),
\end{align}
where $\rho(\mathcal{S}, \mathcal{S}^{\prime})$ indicates the geodesic distance between $\mathcal{S}$ and $\mathcal{S}^{\prime}$.
\end{dfn}

\vspace{3mm}
\begin{lemma}
\label{lemma:subspace projection}
The subspace projection $\omega({\mathcal{S}})$ of $d_1$-dimensional subspace $\mathcal{S}$ onto $d_2(\geq{d_1})$-dimensional subspace $\mathcal{W}$ can be calculated by using the singular value decomposition (SVD) as flows:
\begin{align}
{\rm SVD}: {\mathbf{W}}^{\top} {\mathbf{S}} = {\mathbf{U}}{\mathbf{\Sigma}}{\mathbf{V}}^{\top},\\
\omega({\mathcal{S}}): {\mathcal{S}} \rightarrow {\mathcal{S}}^{\prime} \Rightarrow  {\mathbf{S}} \rightarrow {\mathbf{S}}^{\prime}={\mathbf{WU}},
\end{align}
where ${\mathbf{W}} \in \mathbb{R}^{n{\times}d_2}$ and ${\mathbf{S}} \in \mathbb{R}^{n{\times}d_1}$ are the orthonormal basis matrices of $\mathcal{W}$ and $\mathcal{S}$, respectively.
\end{lemma}

\begin{proof}
Eq.(\ref{min}) can be written as a maximization problem of the following objective function $f$ with respect to the canonical correlation between $\mathcal{W}$ and $\mathcal{S}$:
\begin{align}
f(\mathbf{A}, \mathbf{B})=\underset{\mathbf{a}_i, \mathbf{b}_i} {\operatorname{maximize}} \sum_{i=1}^{d_1} {(\mathbf{W}\mathbf{a}_i)}^{\top} \mathbf{S}\mathbf{b}_i, ~~\text{subject to}~ ||\mathbf{a}_i||=||\mathbf{b}_i||=1,
\end{align}
where $\mathbf{a}_i \in \mathbb{R}^{d_2}$ and $\mathbf{b}_i \in \mathbb{R}^{d_1}$ are the weight coefficient vectors on each base. $\mathbf{W}\mathbf{a}_i$ is a vector in $\mathcal{W}$ and $\mathbf{S}\mathbf{b}_i$ is a vector in $\mathcal{S}$. 

We rewrite the above as follows:
\begin{align}
f(\mathbf{A}, \mathbf{B})= \underset{\mathbf{A}, \mathbf{B}}{\operatorname{maximize}} (\mathbf{WA}) ^{\top} (\mathbf{SB}), ~~\text{subject to}~ ||\mathbf{a}_i||=||\mathbf{b}_i||=1,
\end{align}
where $\mathbf{A}=[\mathbf{a}_1, \dots, \mathbf{a}_{d_1}]$ and $\mathbf{B}=[\mathbf{b}_1, \dots, \mathbf{b}_{d_1}]$. 

Using  the Lagrangian undetermined multiplier method, we obtain the following objective function:
\begin{align}
f(\mathbf{A}, \mathbf{B})&=\rm{tr}\left( \left( \mathbf{WA}\right) ^{\top}{\mathbf{SB}}\right) - \rm{tr}\left( \left( \left( \mathbf{WA}\right) ^{\top}\mathbf{WA}-\mathbf{I}\right) \mathbf{\Sigma} _{1}\right) -\rm{tr}\left( \left( \left( \mathbf{SB}\right) ^{\top} \mathbf{SB}-\mathbf{I}\right) \mathbf{\Sigma} _{2}\right), \\
 &=\rm{tr}\left( \mathbf{A}^{\top}\mathbf{W}^{\top}\mathbf{SB}\right) -\rm{tr}\left( \left( \mathbf{A}^{\top}\mathbf{W}^{\top}\mathbf{WA}-\mathbf{I}\right) \mathbf{\Sigma} _{1}\right) -\rm{tr}\left( \left(\mathbf{B}^{\top}\mathbf{S}^{\top}\mathbf{SB}-\mathbf{I}\right) \mathbf{\Sigma}_{2}\right),
\end{align}
where $\mathbf{\Sigma} _{1} \in \mathbb{R}^{d_1{\times}d_1}$ and $\mathbf{\Sigma} _{2} \in \mathbb{R}^{d_1{\times}d_1}$ are the diagonal matrices representing undetermined multipliers.

To find $\mathbf{A}$ and $\mathbf{B}$ for maximizing the objective function $f$, we differentiate $f(\mathbf{A}, \mathbf{B})$ with respect to $\mathbf{A}$ and $\mathbf{B}$ and set the derivatives to zeros as follows:
\begin{align}
\dfrac{\partial f}{\partial \mathbf{A}} &=\mathbf{W}^{\top} \mathbf{SB}-\mathbf{W}^{\top}\mathbf{WA} \mathbf{\Sigma} _{1}=\mathbf{0},\label{f_a}\\
\dfrac{\partial f}{\partial \mathbf{B}} &=\mathbf{S}^{\top}\mathbf{WA}-\mathbf{S}^{\top}\mathbf{SB} \mathbf{\Sigma} _{2}=\mathbf{0}\label{f_b}.
\end{align}

From Eqs.(\ref{f_a}) and (\ref{f_b}),
\begin{align}
\mathbf{W}^{\top}\mathbf{SB}=\mathbf{W}^{\top}\mathbf{WA\Sigma}_{1}=\mathbf{A\Sigma}_{1},\label{aa}\\
\mathbf{S}^{\top}\mathbf{WA}=\mathbf{S}^{\top}\mathbf{SB\Sigma}_{2}=\mathbf{B\Sigma}_{2}.\label{bb}
\end{align}

Further, from Eq.(\ref{bb}),
\begin{align}
\mathbf{B}=\mathbf{S}^{\top}\mathbf{WA\Sigma}_{2}^{-1}.
\end{align}

By substituting the above equation into Eq.(\ref{aa}), we obtain
\begin{align}
\mathbf{W}^{\top}\mathbf{SS}^{\top}\mathbf{WA\Sigma}_{2}^{-1}&= \mathbf{A\Sigma}_{1},\\
\mathbf{WW}^{\top}\mathbf{SS}^{\top}\mathbf{WA} &=\mathbf{WA\Sigma}_{1}\mathbf{\Sigma}_{2},\\
\mathbf{P}_{1}\mathbf{P}_{2}\left(\mathbf{WA}\right)&=\left(\mathbf{WA}\right) \widehat{\mathbf{\Sigma}},
\end{align}
where  $\mathbf{P}_{1}$ and $\mathbf{P}_{2}$ are the orthogonal projection matrices defined as $\mathbf{WW}^{\top}$ and $\mathbf{SS}^{\top}$, respectively, and $\widehat{\mathbf{\Sigma}}=\mathbf{\Sigma}_{1}\mathbf{\Sigma}_{2}$.
The last equation implies that $\mathbf{WA}$ represents a set of eigenvectos of $\mathbf{P}_{1}\mathbf{P}_{2}$.

Next, we show that $\mathbf{WU}$ also represents a set of the eigevectors of $\mathbf{P}_{1}\mathbf{P}_{2}$.

%\begin{align}
%\mathbf{P}_{1}&=\mathbf{WW}^{T},\\
%\mathbf{P}_{2}&=\mathbf{SS}^{T},\\
%\mathbf{W}^{T}\mathbf{S} &=\mathbf{U\Sigma V}^{T}.
%\end{align}

\begin{align}
\mathbf{P}_{1}\mathbf{P}_{2} &=\mathbf{WW}^{\top}\mathbf{SS}^{\top},\\
\end{align}
By substituting $\mathbf{W}^{\top}\mathbf{S} =\mathbf{U\Sigma V}^{\top}$ into the above equation, we obtain the following:
\begin{align}
\mathbf{P}_{1}\mathbf{P}_{2} &=\mathbf{WU\Sigma V}^{\top}\mathbf{S}^{\top}.\\
\end{align} 

By postmultiplying both sides of the above equation by $\mathbf{W}$, we obtain the following:
\begin{align}
\mathbf{P}_{1}\mathbf{P}_{2}\mathbf{W} &= \mathbf{WU\Sigma V}^{\top}\mathbf{S}^{\top}\mathbf{W},\\
 &=\mathbf{WU\Sigma V}^{\top}\mathbf{V\Sigma U}^{\top},\\
 &=\mathbf{WU\Sigma}^{2}\mathbf{U}^{\top},\\
\mathbf{P}_{1}\mathbf{P}_{2}\left(\mathbf{WU}\right) &=\left(\mathbf{WU}\right) \mathbf{\Sigma}^{2}.\label{p1p2}
 \end{align}

Equation (\ref{p1p2}) denotes that  $\mathbf{WU}$ also represents a set of eigevectors of $\mathbf{P}_{1}\mathbf{P}_{2}$. Thus, as $\mathbf{WU}=\mathbf{WA}$, we conclude that $\mathbf{WU}$ represents the orthonotmal basis of the closest subspace $\mathcal{S}^{\prime} \in \mathcal{W}$ to $\mathcal{S}$.
%, the $d_1$-dimensional subspace $\mathcal{S}^{\prime}$  represented by $\mathbf{WU}$ is also closest to $\mathcal{S}$. 
Consequently, the following relationship can be held:
\begin{align}
\omega(\mathcal{S}) =\underset{\mathcal{S}^\prime{} \in \mathcal{W}} {\operatorname{argmin}}~\rho(\mathcal{S}, \mathcal{S}^{\prime}).
\end{align}

It is worth to confirm that $\mathbf{SV}$ represents a set of eigevectors of $\mathbf{P}_{2}\mathbf{P}_{1}$.

From Eq.(\ref{aa}), we obtain
\begin{align}
\mathbf{A}=\mathbf{W}^{\top}\mathbf{SB\Sigma}_{1}^{-1}.
\end{align}
By substituting the above equation into Eq.(\ref{bb}), we obtain
\begin{align}
\mathbf{S}^{\top}\mathbf{WW}^{\top}\mathbf{SB\Sigma}_{1}^{-1}&=\mathbf{B\Sigma}_{2},\\
\mathbf{SS}^{\top}\mathbf{WW}^{\top}\mathbf{SB} &=\mathbf{SB\Sigma}_{2}\mathbf{\Sigma}_{1},\\
\mathbf{P}_{2}\mathbf{P}_{1}\left(\mathbf{SB}\right)&=\left(\mathbf{SB}\right) \widehat{\mathbf{\Sigma}}.
\end{align}

$\mathbf{P}_{2}\mathbf{P}_{1}$ can be transformed by using $\mathbf{W}^{\top}\mathbf{S} =\mathbf{U\Sigma V}^{\top}$ as follows:
 \begin{align}
\mathbf{P}_{2}\mathbf{P}_{1} &=\mathbf{SS}^{\top}\mathbf{WW}^{\top},\\
&=\mathbf{SV\Sigma U}^{\top}\mathbf{W}^{\top},\\
\mathbf{P}_{2}\mathbf{P}_{1}\mathbf{S} &=\mathbf{SV\Sigma U}^{\top}\mathbf{W}^{\top}\mathbf{S},\\
&=\mathbf{SV\Sigma U}^{\top}\mathbf{U\Sigma V}^{\top},\\
&=\mathbf{SV\Sigma}^{2}\mathbf{V}^{\top},\\
\mathbf{P}_{2}\mathbf{P}_{1}\left(\mathbf{SV}\right) &=\left(\mathbf{SV}\right)\mathbf{\Sigma}^{2}.
\end{align}
The last equation implies that $\mathbf{SV}$ represents a set of eigenvectors of $\mathbf{P}_{2}\mathbf{P}_{1}$.
\end{proof}

\end{document}